%% file: main.tex
\documentclass[10pt]{article} 
\usepackage[preprint]{tmlr}

\input{math_commands.tex}

\usepackage{titletoc}
\usepackage[utf8]{inputenc} 
\usepackage[T1]{fontenc}    
\usepackage{hyperref}       
\usepackage{url}            
\usepackage{booktabs}       
\usepackage{amsfonts}       
\usepackage{nicefrac}       
\usepackage{microtype}      
\usepackage{xcolor}         
\usepackage{amsmath}
\usepackage{amssymb}
\usepackage{mathtools}
\usepackage{adjustbox}
\usepackage{caption}
\usepackage{subcaption}
\usepackage{algpseudocode}
\usepackage{algorithm}
\usepackage{amsthm}
\usepackage{wrapfig}
\usepackage{enumitem}
\usepackage{bbm}

\usepackage{times}
\usepackage{soul}
\usepackage{graphicx}
\usepackage[switch]{lineno}
\usepackage{natbib}
\usepackage{color}

\newcommand\norm[1]{\left\|#1\right\|}
\newtheorem{theorem}{Theorem}
\newtheorem{lemma}{Lemma}
\newtheorem{definition}{Definition}

\title{Modularity aided consistent attributed graph clustering via coarsening}

\author{\name Samarth Bhatia* \email samarth.bhatia23@alumni.iitd.ac.in \\
\addr Indian Institute of Technology, Delhi 
\AND
\name Yukti Makhija* \email yukti.makhija@alumni.iitd.ac.in \\
\addr Indian Institute of Technology, Delhi
\AND
\name Manoj Kumar \email eez208646@iitd.ac.in\\
\addr Indian Institute of Technology, Delhi 
\AND
\name Sandeep Kumar \email ksandeep@ee.iitd.ac.in \\
\addr Indian Institute of Technology, Delhi 
\\\\
\textit{* denotes equal contribution}
}

\def\month{MM}  
\def\year{YYYY} 
\def\openreview{\url{https://openreview.net/forum?id=XXXX}} 

\begin{document}

\maketitle

\vspace{-\baselineskip}
\begin{abstract}

Graph clustering is an important unsupervised learning technique for partitioning graphs with attributes and detecting communities.  However, current methods struggle to accurately capture true community structures and intra-cluster relations, be computationally efficient, and identify smaller communities. We address these challenges by integrating coarsening and modularity maximization, effectively leveraging both adjacency and node features to enhance clustering accuracy. We propose a loss function incorporating log-determinant, smoothness, and modularity components using a block majorization-minimization technique, resulting in superior clustering outcomes. The method is theoretically consistent under the Degree-Corrected Stochastic Block Model (DC-SBM), ensuring asymptotic error-free performance and complete label recovery. Our provably convergent and time-efficient algorithm seamlessly integrates with graph neural networks (GNNs) and variational graph autoencoders (VGAEs) to learn enhanced node features and deliver exceptional clustering performance. Extensive experiments on benchmark datasets demonstrate its superiority over existing state-of-the-art methods for both attributed and non-attributed graphs.


\end{abstract}

\vspace{-10pt}
\section{Introduction}
\vspace{-5pt}


Clustering is an unsupervised learning method that groups nodes together based on their attributes or graph structure, without considering node labels. This versatile approach has numerous applications in diverse fields, such as social network analysis~\citep{social_network_1}, genetics, and bio-medicine~\citep{bioinformatics_clustering, bioinformatics_clustering_2}, knowledge graphs~\citep{KG_ijcai2017p250}, and computer vision~\citep{Mondal_2021_ICCV, caron2018deep}. The wide range of applications has led to the development of numerous graph clustering algorithms designed for specific challenges within these domains.
State-of-the-art graph clustering methods predominantly fall into cut-based, similarity-driven, or modularity-based categories. 


Cut-based methods \citep{Wei1989RatioCut, Shi2000NormalizedCut, Bianchi2020MinCutPool}, aiming to minimize the number of edges (or similar metric) in a cut, may fall short in capturing the true community structure if the cut's edge count doesn't significantly deviate from random graph expectations. This approach originated from the Fiedler vector, which yields a graph cut with a minimal number of edges over all possible cuts \citep{Fiedler1973, Newman2006FindingStructure}. Similarity-based techniques, reliant on pairwise node similarities, group nodes with shared characteristics. These can be computationally intensive and susceptible to noise, yielding suboptimal results, especially in sparse or noisy data scenarios. 
Modularity-based methods rely on a statistical approach, measuring the disparity in edge density between a graph and a random graph with the same degree sequence. These modularity maximization methods exhibit a resolution limit \citep{resolutionlimit_modularity}, as they may inadvertently lead to the neglect of smaller community structures within the graph. Each clustering approach, whether cut-based, similarity-based, or modularity-based, bears its own set of limitations that necessitates careful consideration based on the characteristics of the underlying graph data.

Moreover, graph reduction techniques such as coarsening, summarization, or condensation can also be utilized to facilitate the clustering task \citep{dhillon_multilevel, Manoj2023FGC, ICML-2018-LoukasV-coarsening, JMLR-Loukas2019-coarsening}. In the context of graph coarsening, the objective is to learn a reduced graph by merging similar nodes into supernodes. While this coarsening process is traditionally employed for graph reduction, it can be extended to clustering by reducing the original graph such that each class corresponds to a supernode. However, relying solely on coarsening may lead to suboptimal performance, particularly when the order of the reduction from the original graph size to the number of classes is large and results in a significant loss of information.

In this study, we introduce an optimization-based framework designed to enhance clustering by leveraging both the adjacency and the node features of the graph. Our proposed framework strategically incorporates coarsening and modularity maximization, refining partitioning outcomes and bolstering the effectiveness of the clustering process. The algorithm minimizes a nuanced loss function, Q-MAGC objective, encompassing a $\log\det$ term, smoothness, and modularity components, to ensure efficient clustering. Formulated as a multi-block non-convex optimization problem, our approach is adeptly addressed through a block majorization-minimization technique, wherein variables are updated individually while keeping others fixed. The resulting algorithm demonstrates convergence, showcasing its efficacy in efficiently solving the proposed optimization problem.

To enhance the clustering performance, we embed our Q-MAGC objective function into various Graph Neural Network (GNN) architectures, introducing the Q-GCN algorithm. This novel approach elevates the quality of learned representations by leveraging the message passing and aggregation mechanisms of Graph Convolutional Networks (GCNs), ultimately improving the clustering outcomes. An additional feature of our technique is its capacity to explore inter-cluster relationships. The node attributes derived from the coarsened graph at the conclusion of the process serve as cluster embeddings, shedding light on the distinct characteristics of each cluster. Simultaneously, the edges within the coarsened graph unveil valuable insights into the relationships and connections between different clusters, providing a comprehensive understanding of the overall graph structure. This is a contributing factor to the improvement observed over existing methods and is particularly significant for conducting first-hand analyses on large unlabelled datasets.

We introduce two additional algorithms, Q-VGAE and Q-GMM-VGAE, incorporating variational graph auto-encoders to further enhance clustering accuracy. Through comprehensive experiments, we demonstrate the effectiveness of our proposed algorithms, surpassing the performance of state-of-the-art methods on synthetic and real-world benchmark datasets. Our approach showcases notable improvements in clustering performance, solidifying the robustness and superiority of our proposed methods.

\textbf{Key Contributions.}
\begin{itemize}[leftmargin=*]
    \item We present the first optimization-based framework for attributed graph clustering through coarsening via modularity maximization. Our approach demonstrates efficiency, theoretical convergence, and addresses limitations of existing methods. The paper offers comprehensive analysis and provides theoretical guarantees including KKT optimality, and convergence analysis which are often absent in prior research.
    \item We show that our method is theoretically (weakly and strongly) consistent under a Degree-Corrected SBM (DC-SBM) and shows asymptotically no errors (weakly consistent) and complete recovery of the original labels (strongly consistent).
    \item We demonstrate the seamless integration of the proposed clustering objective with GNN-based architectures, leveraging message-passing in GNNs to enhance the performance of our method. This is also backed up by experiments.~\footnote{Refer to Appendix \ref{supp:code} for the code}
    \item We perform thorough experimental validation on a diverse range of real-world and synthetic datasets, encompassing both attributed and non-attributed graphs of varying sizes. The results demonstrate the superior performance of our method compared to existing state-of-the-art approaches. We want to highlight that our method does not specialise for very large graphs, yet we present preliminary results in this regard.
    \item We conduct ablation studies to evaluate the behavior of the loss terms, compare runtime and complexities, and perform a comprehensive evaluation of modularity.
\end{itemize}

\textbf{Notations.}
Let $G = \{V, E, A, X\}$ be a graph with node set $V = \{v_1,v_2,...,v_p\} \ (|V| = p)$, edge set $E \subset V \times V \} (|E| = e)$, weight (adjacency) matrix $A$ and node feature matrix $X \in \mathbb{R}^{p \times n}$. Also, let $\textbf{d} =A \cdot\textbf{1}_p \in \mathbb{Z}_+^{p}$ be the degree vector, where $\textbf{1}_p$ is the vector of size $p$ having all entries 1.
Then, the graph Laplacian is $\Theta = \textrm{diag}(\textbf{d}) - A$ and the set of all valid Laplacian matrices is defined as:
\(
S_\Theta = \{ \Theta \in \mathbb{R}^{p \times p} | \Theta_{ij} = \Theta{ji} \leq 0 \ \text{for} \ i \neq j, \Theta_{ii} = \sum_{j=1}^{p} \Theta_{ij} \}
\).

\section{Related Works}
\label{sec:related_works}
In this section, we review relevant existing works and highlight their limitations, thereby motivating the need for an improved clustering formulation.

\textbf{Graph Clustering via Coarsening.} Graph coarsening can be extended to graph clustering by reducing the size of the coarsened graph to the number of classes ($k$). However, in most graphs, the number of classes is very small, and reducing the graph to this extent may lead to poor clustering quality.  
DiffPool \citep{Ying2018DiffPool} learns soft cluster assignments at each layer of the GNN and optimizes two additional losses, an entropy to penalize the soft assignments and a link prediction based loss. Next, 
SAGPool \citep{sag_pool} calculates attention scores and node embeddings to determine the nodes that need to be preserved or removed. 
Top-k \citep{topk_pool_icml} also works by sparsifying the graph with the learned weights. MinCutPool \citep{Bianchi2020MinCutPool} formulates a differentiable relaxation of spectral clustering via pooling. However, \citet{DMoN_JMLR_2023} show that MinCutPool's orthogonal regularization dominates over the clustering objective and the objective is not optimized. Some disadvantages of these methods are instability and computational complexity in the case of SAGPool and DiffPool and convergence in MinCutPool. To address these challenges, we need a better loss function to perform the clustering task effectively. 

\textbf{Deep Graph Clustering.} Previous literature can be classified based on contrastive and non-contrastive methods.
On the non-contrastive side, \citet{arva_agrva_ijcai2018p362} proposed ARGA and ARVGA, enforcing the latent representations to align to a prior using adversarial learning. By utilizing an attention-based graph encoder and a clustering alignment loss, \citet{DAEGC_IJCAI} propose DAEGC. 
\citet{DCRN} design the DCRN model to alleviate representation collapse by a propagation regularization term minimizing the Jensen Shannon Divergence (JSD) between the latent and its product with normalized $A$. 
Contrastive methods include AGE \citep{AGE2020} which builds a training set by adaptively selecting node pairs that are highly similar or dissimilar after filtering out high-frequency noises using Laplacian smoothing. 
\citet{gdcl_ijcai2021} propose GDCL to correct the sampling bias by choosing negative samples based on the clustering label. 

VGAEs \citep{kipf2016variational} are an increasingly popular class of GNNs that leverage variational inference \citep{Kingma2014_VAE} for learning latent graph representations in unsupervised settings. They reconstruct the adjacency matrix after passing the graph through an encoder-decoder architecture.
Many attempts have been made to use VGAEs with k-means on latent embeddings, but it has been unsuitable for clustering. This is primarily because embedded manifolds obtained from VGAEs are curved and must be flattened before any clustering algorithms using Euclidean distance are applied. Refer to Appendix Section \ref{supp:vgae_manifolds} for a detailed explanation.
VGAEs only use a single Gaussian prior for the latent space, whereas clustering requires the integration of meta-priors. 
Additionally, the inner-product decoder fails to capture locality and cluster information in the formed edges. 
Several clustering-oriented variants of VGAEs \citep{Mrabah2022RethinkingVGAE, Hui2020_GMM_VGAE} have been developed that overcome most of these challenges. GMM-VGAE \citep{Hui2020_GMM_VGAE} partitions the latent space using a Gaussian Mixture Model and assigns a separate prior for each cluster to better model complex data distributions. Despite the improvement in performance, it's inner-product decoder cannot capture locality information.

\textbf{Modularity Maximization.} Various heuristic algorithms have been established that solve the NP-hard problem of modularity maximization including sampling, simulated annealing \citep{simulated_annealing_modularity, guimera_simulated_annealing}, 
and greedy algorithms (Louvain/Leiden) \citep{greedy_modularity, Blondel_2008}. These algorithms require intensive compute and don't use node features. Modularity maximization using GNNs has also garnered attention recently. DMoN \citep{DMoN_JMLR_2023} optimizes only for modularity with a collapse regularization to prevent the trivial solution, but offers no theoretical guarantees about convergence. Modularity-Aware GAEs and VGAEs \citep{salhagalvan2022modularityawaregae} use a prior membership matrix using Louvain algorithm and optimize for modularity using an RBF kernel as a proxy for same-community assignment. DGCLUSTER \citep{bhowmick-aaai24dgcluster} is a semi-supervised method that makes use of either a subset of labels or pairwise memberships as Auxiliary Information coupled with modularity. 

\section{Background}
\label{background}

In this section, we introduce the concepts that play a key role in the formulation of our method. 

\subsection{Graph Coarsening}
Graph coarsening is a graph dimensionality reduction technique used in large-scale machine learning to construct a smaller or coarsened graph $G_c$ from the original graph $G = \{V, E, A, X\}$ while preserving properties of the original graph $G$. Graph coarsening aims to learn a mapping matrix \( C \in \mathbb{R}_{+}^{p \times k} \), where \( p \) is the number of nodes in the original graph and \( k \) is the number of nodes in the coarsened graph. Each non-zero entry of the mapping matrix \( C \), i.e., \( C_{ij} \), indicates that the \( i \)-th node of $G$ is mapped to the \( j \)-th supernode. Moreover, for a balanced mapping, the mapping matrix \( C \) belongs to the following set \citep{Manoj2023FGC, ICML-2018-LoukasV-coarsening, JMLR-Loukas2019-coarsening}:

\vspace{-\baselineskip}
\begin{align}
    \label{coarsened_features_equation}
    \mathcal{C} =\Big\{ C\in \mathbb{R}_+^{p \times k}|\ \langle C_i, C_j \rangle=0 \ & \forall \; i\neq j,
    \quad \langle C_i, C_i \rangle=d_i, \norm{C_i}_0\geq 1 \ \text{and} \ \norm{[C^{\top}]_i}_0= 1 \Big\}
\end{align}
\vspace{-\baselineskip}

For \( C \in \mathcal{C} \), the relationship between the original graph Laplacian \( \Theta \), the coarsened graph Laplacian \( \Theta_c \), and the mapping matrix \( C \) is given by \( \Theta_c = C^T \Theta C \).



\subsection{Spectral Modularity Maximization}
\label{background:modularity_max}

Spectral Clustering is the most direct approach to graph clustering, where we minimize the volume of inter-cluster edges. Modularity, introduced in \citet{Newman2006ModularityAC}, is the difference between the number of edges within a cluster \( \mathbf{C_i} \) and the expected number of such edges in a random graph with an identical degree sequence. It is mathematically defined as:

\vspace{-\baselineskip}
\begin{align}
    \mathcal{Q} &= \frac{1}{2e} \sum_{i,j=1}^k \Big[A_{ij} - \frac{d_id_j}{2e}\Big]\delta(c_i, c_j)
\end{align}
\vspace{-\baselineskip}

where \( \delta(c_i, c_j) \) is the Kronecker delta between clusters \(i\) and \(j\).
Maximizing this form of modularity is NP-hard \citep{modularity_np_hard}. However, we can approximate it using a spectral relaxation, which involves a modularity matrix \( B \).
The modularity matrix \( B \) and spectral modularity are defined as follows:

\vspace{-\baselineskip}
\begin{align}
    B &= A - \frac{\textbf{dd}^T}{2e}, \qquad
    \textbf{d} = A \cdot \mathbbm{1} , \qquad
    \mathcal{Q} = \frac{1}{2e}Tr(C^TBC)
    \label{modularity_def}
\end{align}
\vspace{-\baselineskip}

\(B\) is symmetric and is defined such that its row-sums and column-sums are zero, thereby making \(\mathbf{1}\) one of its eigenvectors and \(0\) the corresponding eigenvalue. These spectral properties of the modularity matrix are also observed in the Laplacian, as noted in \citet{Newman2006ModularityAC}, which is a crucial element in spectral clustering. Modularity is maximized when \(u_1^T s\) is maximized, where \(u\) are the eigenvectors of \(B\) and \(s\) is the community assignment vector, i.e., placing the majority of the summation in \(Q\) on the first (and largest) eigenvalue of \(B\). Moreover, modularity is closely associated with community detection. These special spectral properties make \(B\) an ideal choice for graph clustering. While modularity maximization has been extensively studied, heuristic algorithms for it are computationally intensive, such as the Newman-Girvan algorithm with $\mathcal{O}(p^3)$ time complexity. The Louvain/Leiden algorithms \citep{greedy_modularity, Blondel_2008} improve on this.

\subsection{Motivation and Proposed Problem Formulation}
Current graph clustering methods often struggle to accurately capture both intra-cluster and inter-cluster relationships, achieve computational efficiency, and identify smaller communities, thereby limiting their effectiveness. These methods face significant challenges including instability and the lack of convergence guarantees, especially when employing coarsening techniques. Modularity maximization, despite extensive study, relies on computationally intensive heuristics and lacks theoretical convergence guarantees. To address these issues, we propose an optimization-driven framework that strategically integrates coarsening and modularity. By incorporating both adjacency and node features, our approach aims to robustly capture intra-cluster and inter-cluster dynamics. This framework ensures stability, guarantees convergence, and consistently delivers superior clustering results with enhanced computational efficiency compared to existing methods. Given original graph $G(V,E,A,X)$, the proposed generic optimization formulation is:
\begin{align}
    \label{mainoptimization}
    \underset{ C}{\min}\;  \mathcal{L}_{MAGC} &= \  f(C, X, \Theta) + g(C, A) + h(C, \Theta)\nonumber\\ 
    \text{subject to}\ C &\in \mathcal{S_C}=\{C\in \mathbb{R}^{p \times k}|C \geq 0,\ \norm{[C^T]_i}_2^2 \leq 1\}\;\ \forall \ i =1,2 \ldots p
\end{align}
\vspace{-\baselineskip}

Here, \( C \in \mathbb{R}_+^{p \times k} \) represents the clustering matrix to be learned, where each non-zero entry \( C_{ij} \) indicates that the \( i \)-th node is assigned to the \( j \)-th cluster. The term \( f(C, X, \Theta) \) denotes the graph coarsening objective, managing inter-cluster edges. The function \( g(C, A) \) represents the modularity objective, improving clustering performance. Additionally, \( h(C, \Theta) \) is a regularization term enforcing desirable properties in the clustering matrix, as defined in Equation \ref{coarsened_features_equation}. The overarching goal of the optimization problem \ref{mainoptimization} is to choose \( f(C, X, \Theta) \), \( g(C, A) \), and \( h(C, \Theta) \) such that nodes are optimally clustered and inter-cluster connectivity is maintained. In the next section, we will develop the clustering algorithm, leveraging the coarsening objective and modularity to enhance clustering performance.

\section{Proposed Method}

\label{method}

Given a graph $G = \{V, E, A, X\}$, considering $f(C,X_C,\Theta) =\text{tr}(X_C^TC^T\Theta C X_C)$, 
$g(C,A) = - \frac{\beta}{2e}\textrm{tr}(C^TBC)$ (where $B = A - \frac{\textbf{dd}^T}{2e}, \quad \textbf{d} = A \cdot \mathbbm{1}$), and  
$h(C, \Theta) = - \gamma \log\det(C^T\Theta C + J)$, to obtain clustering matrix $C$ we formulate the following optimization problem : 

\vspace{-\baselineskip}
\begin{align}
    \label{optimization_objective}
    \underset{X_C, C}{\min}\;  \mathcal{L}_{MAGC} &= \ \textrm{tr}(X_C^TC^T\Theta CX_C)  - 
    \frac{\beta}{2e}\textrm{tr}(C^TBC) - \gamma \log\det(C^T\Theta C + J) \nonumber\\ 
    \text{subject to}\ C \in \mathcal{S_C} &= \{C\in \mathbb{R}^{p \times l}\big| \ \norm{C^T_i}_2^2 \leq 1\}\;\forall i, X=CX_C\; \; 
    \text{where, } J = \frac{1}{k} \textbf{1}_{k\times k}
\end{align}


The term \(\textrm{tr}(X_C^T C^T \Theta CX_C)\) represents the smoothness or Dirichlet energy of the coarsened graph while $C^T\Theta C$ is the Laplacian matrix of the coarsened graph. Minimizing smoothness ensures that the clusters or supernodes with similar features are linked with stronger weights. The term \(\textrm{tr}(C^T BC)\) corresponds to the graph's modularity, enhancing the quality of the clusters formed. The term \(-\log\det(C^T\Theta C + J)\) is crucial for maintaining inter-cluster edges in the coarsened graph. For a connected graph matrix with \(k\) super-nodes or clusters, the rank of \(C^T\Theta C\) is \(k-1\). Adding \(J\) to \(C^T\Theta C\) makes \(C^T\Theta C + J\) a full-rank matrix without altering the row and column space of \(C^T\Theta C\) \citep{kumar2020unified}.


Problem \ref{optimization_objective} is a multi-block non-convex optimization problem solved using the Block Successive Upper-bound Minimization (BSUM) framework. All terms except modularity are convex in nature, which we prove in Appendix \ref{supp:convexity_proof}.
We iteratively update \( C \) and \( X_C \) alternately while keeping the other constant. This process continues until convergence or the stopping criteria are met.
Since the constraint $X = CX_C$ is hard and difficult to enforce, we relax it by adding the term \(\frac{\alpha}{2}\|X - CX_C\|_F^2\) to the objective.
This term ensures that each node is assigned to a cluster, leaving no node unassigned.
Additionally, when needed, we can add an optional sparsity regularization term \(\lambda\|C^T\|_{1,2}^2\), which can be seen as ensuring that each node is assigned to exactly one cluster, avoiding any overlap in node assignments across clusters.


\textbf{Update rule of C.}
Treating $X_C$ as constant and $C$ as a variable the sub-problem for $C$ is:
\begin{align}
    \label{problem in C}
    \underset{ C}{\min}\;  f(C) &= \ \textrm{tr}(X_C^TC^T\Theta CX_C)  - 
    \frac{\beta}{2e}\textrm{tr}(C^TBC) - \gamma \log\det(C^T\Theta C + J) +\frac{\alpha}{2}\|X-CX_C\|_F^2 \nonumber\\ 
    \text{subject to}\ C &\in \mathcal{S_C}=\{C\in \mathbb{R}^{p \times l}|C \geq 0,\ \norm{C^T_i}_2^2 \leq 1\}\;\forall i, \; \; 
    \text{where, } J = \frac{1}{k} \textbf{1}_{k\times k}
\end{align}
By using the first-order Taylor series approximation, a majorised function for $f(C)$ at $C^{t}$ ($C$ after $t$ iterations) can be obtained as:
\begin{align}
    g(C|C^t) &= f(C^t) + (C - C^t)\nabla f(C^t) + \frac{L}{2}\norm{C - C^t}^2 
\end{align}
where $f(C)$ is $L-$Lipschitz continuous gradient function $L=\max(L_1,L_2, L_3,L_4)$ with $L_1,L_2, L_3,L_4$ the Lipschitz constants of $-\gamma \text{log det}(C^T\Theta C +J)$, $\text{tr}(X_C^{T}C^T\Theta CX_C)$, $\|CX_C - X\|^2_F$, $\text{tr}(C^TBC)$,  respectively. We prove this in Appendix \ref{supp:optimal_soln}. After ignoring the constant term, the majorised problem of \eqref{problem in C} is
\begin{align}
    \underset{C \in \mathcal{S}_C}{\min}&\ \frac{1}{2} C^TC - C^T\Big(C^t - \frac{1}{L}\nabla f(C^t)\Big) \label{majorized_objective}
\end{align}
\Eqref{majorized_objective} is the majorized problem of \eqref{optimization_objective}.
The optimal solution to \eqref{majorized_objective}, found by using Karush–Kuhn–Tucker (KKT) optimality conditions is (Proof is deferred to the Appendix \ref{supp:optimal_soln}): 
\begin{align}
    \label{eqn:C_Update}
    C^{t+1} &= \Big(C^{t} - \frac{1}{L}\nabla f\big(C^t\big) \Big)^{+} \\ 
    \text{where, } \nabla f\big(C^t\big) &= -2\gamma\Theta C^t({C^t}^T\Theta C^t + J)^{-1} +  
     \alpha(C^tX_C - X)X_C^T + 2\Theta C^tX_CX_C^T  - \frac{\beta}{e}BC^t
\end{align}
\textbf{Update rule of $X_C$.}
Treating $C$ fixed and $X_c$ as variable. The subproblem for updating $X_c$ is
\begin{align}\label{UpdatextildeCGL}
\begin{array}{ll}
\underset{\tilde{X}}{\min} f(\tilde{X})=\textrm{tr}(X_C^TC^T\Theta CX_C) + \frac{\alpha}{2}\|X-CX_C\|_F^2
\end{array}
\end{align}
The closed form solution of problem \eqref{UpdatextildeCGL} can be obtained by putting the gradient of $f(\tilde{X})$ to zero. 
\begin{align}   
    &X_C^{t+1} = \Big(\frac{2}{\alpha}C^T\Theta C + C^TC \Big)^{-1}C^TX
    \label{eqn:X_Update}
\end{align}
\vspace{-\baselineskip}

\begin{wrapfigure}{r}{0.5\linewidth}
\vspace{-2\baselineskip}
    \begin{minipage}{\linewidth}
        \begin{algorithm}[H]
            \caption{Q-MAGC Algorithm}
            \begin{algorithmic}[1]
            \Require $G(X, \Theta), \alpha, \beta, \gamma, \lambda$
            \State $t \gets 0$
            \While{Stopping Criteria not met}
                \State Update $C^{t+1}$ as in Eqn. \ref{eqn:C_Update}
                \State Update $X_C^{t+1}$ as in Eqn. \ref{eqn:X_Update}
                \State $t \gets t + 1$
            \EndWhile \\
            \Return $C^t, X_C^t$
            \end{algorithmic}
        \label{alg:q-magc}
        \end{algorithm}
    \end{minipage}
\vspace{-2\baselineskip}
\end{wrapfigure}


\textbf{Convergence Analysis.}
\begin{theorem}
\label{theorem:convergence}
The sequence \( \{C^{t+1}, X_C^{t+1}\} \) generated by Algorithm \ref{alg:q-magc} converges to the set of Karush–Kuhn–Tucker (KKT) optimality points for Problem \ref{optimization_objective}
\end{theorem}
\begin{proof}
The detailed proof can be found in the Appendix Section \ref{supp:convergence}.
\end{proof}

\textbf{Complexity Analysis.}
The worst-case time complexity of a loop (i.e. one epoch) in Algorithm \ref{alg:q-magc} is $\mathcal{O}(p^2k + pkn)$ because of the matrix multiplication in the update rule of $C$ \eqref{eqn:C_Update}. Here, $k$ is the number of clusters and $n$ is the feature dimension. Note that $k$ is much smaller than both $p$ and $n$. This makes our method much faster than previous optimization based methods and faster than GCN-based clustering methods which have complexities around $\mathcal{O}(p^2n + pn^2)$. We discuss this more in Appendix Section \ref{supp:common_complexities}.

\textbf{Consistency Analysis on Degree Corrected Stochastic Block Models (DC-SBM).}

To check whether the proposed objective (Equation \ref{optimization_objective}) results in consistent clustering, we need to assume a random graph. A graph \( G(V, E) \) is generated by a DC-SBM with \( p \) nodes which belong to \( k \) classes.
Following the setup in \citet{zhao2012_consistency}, each node \( v_i \) is associated with a label-degree pair \((y_i, t_i)\) drawn from a discrete joint distribution \(\Pi_{k \times m}\). 
Here, \( t_i \) takes values \( 0 \leq x_1 \leq \dots \leq x_m \) with \(\mathbb{E}[t_i] = 1\), and \( y_i \in [k] \). 

Additionally, we have a symmetric \( k \times k \) matrix \( P \) that specifies the probabilities of inter-cluster edges.
The edges between each node pair \((v_i, v_j)\) are sampled as independent Bernoulli trials with probability \(\mathbb{E}[A_{ij}] = t_i t_j P_{y_i y_j}\).
To ensure \(\mathbb{E}[A_{ij}] < 1\), DC-SBMs must satisfy \( x_m^2 (\max_{i,j} P_{ij}) \leq 1 \).
The matrix \( P \) is allowed to scale with the number of nodes \( p \) and is reparameterized as \( P_p = \rho_p P \), where \(\rho_p = \Pr[A_{ij} = 1] \rightarrow 0\) as \( p \) increases, and \( P \) is fixed.
The average degree of the DC-SBM graph is defined as \(\lambda_p = p \rho_p\).

Let \(\Hat{y}_i\) denote the predicted cluster for node \( v_i \). The assignment matrix \( C \) is the one-hot encoding of \( y \), such that \( C_i = \text{one-hot}(y_i) \).

We define \( O(e) \in \mathbb{Z}_{\geq 0}^{k \times k} \) as the inter-cluster edge count matrix for a given cluster assignment \( e \in [k]^p \). Here, \( O_{ql}(e) = \sum_{ij} A_{ij} \mathbbm{1}\{e_i = q, e_j = l\} \) denotes the number of edges between clusters \( q \) and \( l \). Also, We represent the class frequency distribution as \( \pi \in [0,1]^k \), where \( \pi_i \) indicates the fraction of nodes assigned to cluster \( i \).
\begin{definition}
\label{def:weak_strong_consistency}
    (Strong and Weak Consistency). The clustering objective is defined to be strongly consistent if $ \lim \limits_{p \to \infty} Pr[\hat{y} = y] \to 1
$. A weaker notion of consistency is defined by 
$
    \lim\limits_{p \to \infty} Pr\bigg[\frac{1}{p} \sum\limits_{i=1}^p \mathbbm{1}\{\hat{y} \neq y\} < \epsilon \bigg] \to 1 \ \ \forall \ \epsilon > 0.
$
\end{definition}

\citet{zhao2012_consistency} prove the consistency of multiple clustering objectives under the DC-SBM, including Newman–Girvan modularity (Theorem 3.1). They also provide a general theorem (Theorem 4.1) for checking consistency under DC-SBMs for any criterion \(\mathcal{L}(e)\), which can be expressed as \(\mathcal{L}(e) = F\left(\frac{O(e)}{\mu_p}, \pi\right)\), where \(\mu_p = p^2 \rho_p\) and \( e \in [k]^p \) is a cluster assignment. 

The consistency of an objective is evaluated by determining if the population version of \(\mathcal{L}(e)\) is maximized by the true cluster assignment \( y \). The population version of \(\mathcal{L}(e)\) is obtained by taking its conditional expectations given \( y \) and \( t \). We consider an array \( S \in \mathbb{R}^{k \times k \times m} \) and define a matrix \( H(S) \in \mathbb{R}^{k \times k} \) as
\( H_{kl}(S) = \sum_{abuv} x_u x_v P_{ab} S_{kau} S_{lbv} \). 
Additionally, we define a vector \( h(S) \in \mathbb{R}^k \) as
\( h_k(S) = \sum_{au} S_{kau}\). Here, \( H(S) \) and \( h(S) \) denote the population versions of \( O(e) \) and \( \pi \), respectively, such that \(\frac{1}{\mu_p} \mathbb{E}[O | C, t] = H(S)\) and \(\mathbb{E}[\pi | C, t] = h(S)\).
\begin{lemma}
    \label{lemma:consistency_thm_4.1}
    (Theorem 4.1 from \cite{zhao2012_consistency}) For any $\mathcal{L}(e) = F(\frac{O(e)}{\mu_p}, f(e))$, if $F$ is uniquely maximized by $S = \mathbb{D}$ and $\pi, P, F$ satisfy the regularity conditions, then $\mathcal{L}$ is strongly consistent under DC-SBMs if $\frac{\lambda_p}{log p} \rightarrow \infty$ and weakly consistent if $\lambda_p \rightarrow \infty$:
    \begin{itemize}[noitemsep]
    \item[1.] $F$ is Lipschitz in its arguments ($H(S), h(S)$) 
    \item[2.] The directional derivative $\frac{\partial^2}{\partial\varepsilon^2}F(M_0 + \varepsilon(M_1 - M_0), \mathbf{t_0} + \varepsilon(\mathbf{t_1} - \mathbf{t_0}))\big|_{\varepsilon=0^+}$ is continuous in $(M_1, \mathbf{t_1})$ for all $(M_0, \mathbf{t_0})$ in a neighborhood of $(H(\mathbb{D}), \pi)$ 
    \item[3.] With $G(S) = F(H(S), h(S))$, $\quad \frac{\partial G ((1-\varepsilon)\mathbb{D} + \varepsilon S)}{\partial \varepsilon}|_{\varepsilon=0^+} < -C < 0 \ \forall\ \pi, P$ 
    \end{itemize}
\end{lemma}
It is important to note that the paper focuses on maximizing an objective function. Therefore, we consider the negative of our loss function, \(-\mathcal{L}_{MAGC}\). Additionally, we only consider the objective function defined in \eqref{optimization_objective}, as regularizers can be adjusted for different downstream tasks.
\begin{theorem}
\label{theorem2:consistency}
    Under the DC-SBM, $\mathcal{L}_{MAGC}$ is strongly consistent when $\lambda_p/\log p \rightarrow \infty$ and weakly consistent when $\lambda_p \rightarrow \infty$.
\vspace{-10pt}
\end{theorem}
\begin{proof}
The proof is divided into two parts. First, we demonstrate that \(\mathcal{L}_{MAGC}\) can be expressed in terms of \(O\) and \(\pi\). Second, we verify that, in this form, \(\mathcal{L}_{MAGC}\) satisfies the regularity conditions specified in Theorem 4.1 of \cite{zhao2012_consistency}. It is important to note that DC-SBMs do not consider node features, so we treat \(X\) and \(X_C\) as constants.
We start by noting that \( O = C^TAC \), i.e., the edge count matrix \( O \) is the class-transformed version of the adjacency matrix \( A \). This has a good intuition behind it, as \( C \) is the mapping from nodes to classes, and it can also be proved easily.
\begin{equation}
\{C^TAC\}_{ql} = (C^T)_q A C_{:l} = \sum_{i=1}^p \sum_{j=1}^p (C^T)_{qi}A_{ij} C_{jl} = \sum_{i=1}^p \sum_{j=1}^p C_{iq}A_{ij} C_{jl}
\end{equation}
\(C_{iq}\) equals 1 only when node \(i\) is in cluster \(q\), and \(C_{jl}\) equals 1 only when node \(j\) is in cluster \(l\). Therefore, this expression simplifies to the definition of \(O\).
\begin{equation}
\{C^TAC\}_{ql} = \sum_{i=1}^p \sum_{j=1}^p A_{ij}\mathbbm{1}\{y_i=q,y_j=l\}
\end{equation}
This can be viewed as the \emph{cluster-level adjacency matrix}, which quantifies the edges between clusters. Next, we formalize the relationship between \(O\) and \(C^T D C\), where \(D = \text{diag}(d)\) is the degree matrix.
\vspace{-\baselineskip}
\begin{equation}
    \{C^TDC\}_{ql} = \sum_{i=1}^p \sum_{j=1}^p C_{iq}D_{ij}C_{jl} = \sum_{i=1}^p C_{iq}d_i C_{il} \ \ \ [ D_{ij} = 0 \text{ if } i \neq j] \\
\end{equation}
Since the rows of \(C\) are orthonormal, \(C_{iq}C_{il} = 1\) if and only if \(q = l\). Thus, we obtain the following expression for the diagonal entries of \(C^TDC\):
\vspace{-\baselineskip}
\begin{equation}
    \{C^TDC\}_{qq} = \sum_{i=1}^p C_{iq}^2 d_i = \sum_{i=1}^p C_{iq} d_i \ \ \ \bigg[\because C_{ij}\in \{0,1\} \bigg] \\
\end{equation}
Now, we consider $\sum_{l=1}^{k} \{C^TAC\}_{ql}$
\vspace{-\baselineskip}
\begin{equation}
    = \sum_{i=1}^p \sum_{j=1}^p\bigg( C_{iq}A_{ij} \sum_{l=1}^k C_{jl}\bigg)
    = \sum_{i=1}^p \sum_{j=1}^p C_{iq}A_{ij}\cdot 1 = \sum_{i=1}^p C_{iq}d_{i}
\end{equation}
Thus, we showed that $C^TDC = \text{diag}(\sum_{l=1}^k [C^TAC]_k)$, which intuitively represents the \emph{cluster-level degree matrix} because it is the (diagonal of) row/column-sum of the $C^TAC$ or cluster-level adjacency matrix. 
We now express the Laplacian of the coarsened graph as a function of $O$, $\Theta_C(O) = C^T\Theta C = C^T(D - A) C =\text{diag}(\sum_{l=1}^kO_l) - O$. This implies that the terms $\textrm{tr}(X_C^TC^T\Theta CX_C)$, $\textrm{tr}(C^TBC)$, and $\log\det(C^T\Theta C + J)$ in our clustering objective are functions of $\Theta_C(O)$. Using Theorem 4.1 from \cite{zhao2012_consistency}, we conclude that $\mathcal{L}_{\text{MAGC}}$ must be uniquely minimized at any point $(y^*, A^*)$ s.t. $\mathbb{E}_p[\pi(y_p)] = \pi(y^*)$ and $\mathbb{E}_p[A^{(p)}] = A^*$. 
The Laplacian $\Theta$ can be written as a function of the adjacency $\Theta_{ij} = \sum_j A^*_{ij} - A^*_{ij}$ (assuming no self-loops).

For the modularity term, this has been proven in Theorem 3.1 of \citet{zhao2012_consistency}. 
For the constraint relaxation term $\frac{1}{2}\|CX_C - X\|_F^2$, we see that $X_C$ is a function of only $C(y^*)$ and $\Theta(A^*)$ from it's update rule \eqref{eqn:X_Update}, so we get $\text{term}_\alpha(y^*, \Theta) = \frac{1}{2}\| C(y^*)X_C(y^*, A^*) - X\|_F^2$.
$\text{tr}(X_CC^T\Theta CX_C)$ is already just a function of $X_C, C, \Theta$ which we have shown above to be functions of only $y^*, A^*$.
By the same logic, $-\log\det (C^T\Theta C + J)$ is also a function of only $y^*,A^*$ ($J$ is a constant = $\frac{1}{k}\mathbf{1}_{k\times k}$). Thus, $\mathcal{L}_{\text{MAGC}}$ can be written as required above.


Next, we compute the population version of $\mathcal{L}_{MAGC}$.
\begin{align}
F(H(S), h(S)) &= \mathbb{E}[-\mathcal{L}_{MAGC}|C, t] = \underbrace{f_1(H(S), h(S))}_{\text{smoothness}} + \underbrace{f_2(H(S), h(S))}_{\text{log-determinant}} + \underbrace{f_3(H(S), h(S))}_{\text{modularity}}\\
\text{where, } f_1(H(S), h(S)) &= \mu_p \text{tr}(X_C^T [H(S) - \text{diag}(\sum_{j=1}^{k} H(S)_{ij})] X_C) \\
f_2(H(S), h(S)) &= \frac{\mu_p}{b} \text{tr}\bigg(H(S) - \text{diag}(\sum_{j=1}^{k} H(S)_{ij}) \bigg) - \frac{1}{b} + k - k\log b \\
f_3(H(S), h(S)) &= \frac{\text{tr}(H(S))}{\tilde{P_0}} - \frac{\sum_{i=1}^k (\sum_{j=1}^k H(S)_{ij})^2}{\tilde{P_0}^2} \\
\text{and, } \tilde{P_0} &= \sum_{ab} \tilde{\pi}_a \tilde{\pi}_b P_{ab} \quad ; \quad
\Tilde{\pi}_a = \sum_u x_u \Pi_{au} \text{ with } \sum_a \Tilde{\pi}_a = 1, \text{ since } \mathbb{E}[t_i] = 1.
\end{align}
For brevity, we denote $F(H(S), h(S))$ by $F(S)$. Now for the second part of the proof, it is easy to prove that $F(S)$ satisfies the regularity conditions stated in Lemma \ref{lemma:consistency_thm_4.1}. The detailed proof is deferred to Appendix \ref{supp:proof_consistency_conditions}.
This concludes the proof that $\mathcal{L}_{MAGC}$ is strongly and weakly consistent under the DC-SBM.
\end{proof}
We also demonstrate experimental consistency of the optimization objective \eqref{optimization_objective} in Appendix \ref{supp:sbm}, resulting in complete recovery of the labels. This consistency extends to the objective with additional regularization terms, as detailed therein. We want to emphasize that the consistency analysis for Equation \ref{optimization_objective} holds regardless of how the objective is optimized: whether by integrating our loss into Graph Neural Networks (GNNs) or by employing Block Majorization-Minimization techniques.


\section{Integrating with GNNs}

\begin{wrapfigure}{r}{0.5\linewidth}
    \vspace{-5\baselineskip}
    \begin{minipage}{\linewidth}
        \begin{algorithm}[H]
            \caption{Q-GCN Algorithm}
            \begin{algorithmic}[1]
            \Require $G(X, \Theta), \alpha, \beta, \gamma, \lambda$
            \Require $k$ GCN layers - weights $W^{(1..k)}$
            \State $t \gets 0$
            \While{Stopping Criteria not met}
                \State $C^{t+1} \gets \text{GCN}(X, A)$
                \State $X_C^{t+1} \gets C^\dagger X$
                \State $t \gets t + 1$
                \State $\mathcal{L} \gets \mathcal{L}_{\text{MAGC}}(C, X)$
            \EndWhile \\
            \Return $C^t, X_C^t$
            \end{algorithmic}
        \label{alg:q-gcn}
        \end{algorithm}
    \end{minipage}
    \vspace{-\baselineskip}
\end{wrapfigure}

Our optimization framework integrates seamlessly with Graph Neural Networks (GNNs) by incorporating the objective \eqref{optimization_objective} into the loss function. This integration can be minimized using gradient descent. We demonstrate the effectiveness of this approach on several popular GNN architectures, including Graph Convolutional Networks (GCNs) \citep{kipfandwelling17_GCN}, Variational Graph Auto-Encoders (VGAEs) \citep{kipf2016variational}, and a variant known as Gaussian Mixture Model VGAE (GMM-VGAE) \citep{Hui2020_GMM_VGAE}.


We iteratively learn the matrix \( C \) using gradient descent. From \eqref{coarsened_features_equation}, we update \( X_C \) using the relation \( X_C =C^\dagger X\). It is important to note that the loss term \( CX_C - X \) will not necessarily be zero. This arises because we use a "soft" version of \( C \) ($C_{i,j}$ is the probability $\in[0, 1]$) in the loss function to enable gradient flow, while a "hard" version of \( C \) ($C_{i,j}$ is the binary assignment $\in \{0,1\}$) is used in the update step. This method ensures that \( C \) naturally becomes harder and exhibits a higher prediction probability.

\textbf{Q-GCN.} We integrate our loss function (\(\mathcal{L}_{MAGC}\), as defined in \eqref{optimization_objective}) into a simple three-layer Graph Convolutional Network (GCN) model. The soft cluster assignments \(C\) are learned as the output of the final GCN layer. The detailed architecture and loss function are illustrated in Figure \ref{fig:architecture}.


\textbf{Q-VGAE.} 
The VGAE loss can be written as
\(
    \mathcal{L}_{VGAE} = \lambda_{recon}\underbrace{\mathbb{E}_{q(Z|X,A)} [\log p(\hat{A}|Z)]}_{\text{Reconstruction Error}} 
    - \lambda_{kl}\underbrace{\textrm{KL}[q(Z|X,A)\ ||\ p(Z)]}_{\text{Kullback-Leibler divergence}}
\)    

\begin{wrapfigure}{r}{0.5\linewidth}
    \vspace{-2\baselineskip}
    \begin{minipage}{\linewidth}
    \begin{algorithm}[H]
        \caption{Q-VGAE/Q-GMM-VGAE Algorithm\\ $Z$ is the latent space of the VGAE and $\hat{A}$ is the reconstructed adjacency matrix.}
        \begin{algorithmic}[1]
        \Require $G(X, \Theta), \alpha, \beta, \gamma, \lambda$
        \Require Variational Encoder - VarEnc (GCN, $\mu$, $\sigma$)
        \Require Prediction Head - Pred (GCN or GMM)
        \State $t \gets 0$
        \While{Stopping Criteria not met}
            \State $Z \gets \text{VarEnc}_{\mu,\sigma}(X, A)$
            \State $\hat{A} \gets ZZ^T$
            \State $C^{t+1} \gets \text{Pred}(Z, A)$
            \State $X_C^{t+1} \gets C^\dagger X$
            \State $t \gets t + 1$
            \State $\mathcal{L} \gets \mathcal{L}_{\text{MAGC}}(C, X) + \mathcal{L}_{\text{VGAE}}(X, A, Z, \hat{A})$
        \EndWhile \\
        \Return $C^t, X_C^t$
        \end{algorithmic}
    \label{alg:q-vgae}
    \end{algorithm}
    \end{minipage}
    \vspace{-2\baselineskip}
\end{wrapfigure}

where, \(Z\) represents the latent space of the VGAE, \(\hat{A}\) is the reconstructed adjacency matrix, and \(\lambda_{recon}\) and \(\lambda_{kl}\) are hyperparameters. We add a GCN layer on top of this architecture, which takes \(Z\) as input and predicts \(C\). A detailed summary of the VGAE loss terms is provided in Appendix Section \ref{supp:VGAE_eqns}.
For the VGAE, we minimize the sum of three losses: the reconstruction loss, the KL-divergence loss, and our loss. This combined loss function is expressed as:
\(
    \mathcal{L}_{Q-VGAE} = \mathcal{L}_{MAGC} + \mathcal{L}_{VGAE}
\)

\textbf{Q-GMM-VGAE.} 
This variant of VGAE incorporates a Gaussian Mixture Model (GMM) in the latent space to better capture data distributions. This approach is effective because it minimizes the evidence lower bound (ELBO) or variational lower bound \citep{Hui2020_GMM_VGAE, Kingma2014_VAE, kipf2016variational} using multiple priors, rather than a single Gaussian prior as in standard VGAE. \citet{Hui2020_GMM_VGAE} use a number of priors equal to the number of clusters.

\begin{figure*}[t]
    \centering
    \includegraphics[width=\linewidth]{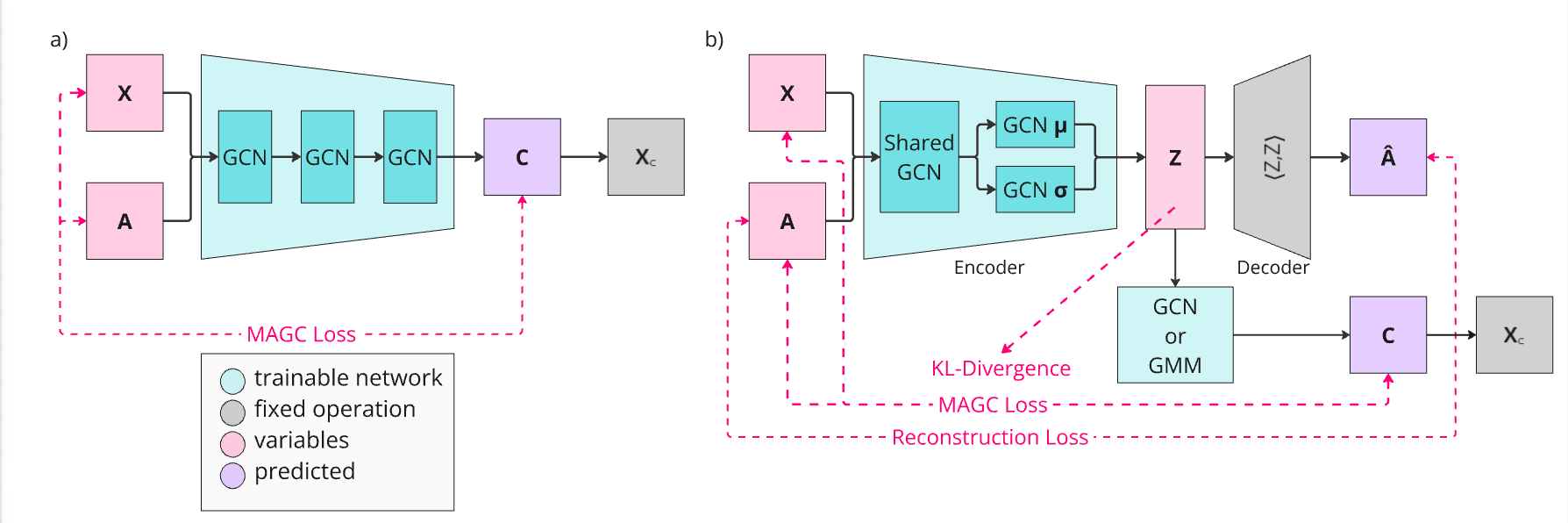}
    \caption{\textbf{a)} \textbf{ Architecture of Q-GCN.} We want to train the encoder to learn the soft cluster assignment matrix \( C \). The coarsened features \( X_C \) are obtained using the relation \( X_C^{t+1} = {C^{t+1}}^\dagger X \). Finally, our proposed MAGC loss is then computed using \( C \) and \( X_C \). \\
    \textbf{b)} \textbf{Architecture of Q-VGAE/Q-GMM-VGAE.} The three-layer GCN encoder takes \( X \) and \( A \) as inputs to learn the latent representation \( Z \) of the graph. \( Z \) is then passed through an inner-product decoder to reconstruct the adjacency matrix \( \hat{A} \). The reconstruction loss is calculated between \( \hat{A} \) and \( A \), and the KL-divergence is applied to \( Z \). In Q-VGAE (or Q-GMM-VGAE), \( Z \) is also passed through a GCN layer (or GMM) to output the soft cluster assignments \( C \). The MAGC loss is then computed in the same manner as in Q-GCN.} 
    \label{fig:architecture}
    \vspace{-\baselineskip}
\end{figure*}

\section{Experiments}
\label{experiments}

\subsection{Benchmark Datasets and Baselines}
We evaluate our method on a diverse set of datasets, including small attributed datasets like Cora and CiteSeer, larger datasets like PubMed, and unattributed datasets such as Airports (Brazil, Europe, and USA). Additionally, we test our method on very large datasets like CoauthorCS/Physics, AmazonPhoto/PC, and ogbn-arxiv. A detailed summary of all the datasets used is provided in Appendix \ref{supp:datasets}.

To assess the performance of our method, we compare it against three types of state-of-the-art methods based on the input and architecture type: methods that use only node attributes, methods that use only graph structure, and methods that use both graph structure and node attributes. The last category is further subdivided into graph coarsening methods, GCN-based architectures, VGAE-based architectures and contrastive methods, and heavily modified VGAE architectures. This comprehensive evaluation allows us to demonstrate the robustness and versatility of our approach across various data and model configurations.

\subsection{Metrics}

To evaluate the performance of our method, we utilize label alignment metrics that compare ground truth node labels with cluster assignments. Specifically, we measure Normalised Mutual Information (NMI), Adjusted Rand Index (ARI), and Accuracy (ACC), with higher values indicating superior performance. For detailed explanations of these metrics, please refer to Appendix \ref{supp:datasets}.

We selected Normalized Mutual Information (NMI) as the primary metric for evaluating model performance based on its prominence in graph clustering literature and its comprehensive ability to assess the quality of our cluster assignments.

\textbf{Training Details.}
Training details are available in the Appendix \ref{supp:training_details}.

\subsection{Attributed Graph Clustering}

We present our key results on the real datasets Cora, CiteSeer, and PubMed in Table \ref{table:attributed_results}. Our proposed method outperforms all existing methods in terms of NMI and demonstrates competitive performance in Accuracy and ARI. The best models were selected based on NMI scores. Results for very large datasets are provided in Appendix \ref{supp:large_dataset_results}. Unlike some methods such as S3GC \citep{s3gc_devvrit2022}, which use randomly-sampled batches that can introduce bias by breaking community structure, we perform full-batch training by passing the entire graph. For extremely large graphs, such as ogbn-arxiv, we have also utilized batching.

\begin{table*}[t!]
    \centering
    \begin{adjustbox}{width=\linewidth}
    \begin{tabular}{lccc@{\hskip 20pt}ccc@{\hskip 20pt}ccc}
    \toprule[1.5pt]
    & \multicolumn{3}{c}{Cora} & \multicolumn{3}{c}{CiteSeer} & \multicolumn{3}{c}{PubMed} \\
    \cmidrule(r){2-4} \cmidrule(r){5-7} \cmidrule(r){8-10}
    \textbf{Method} & \textbf{ACC $\uparrow$} & \textbf{NMI $\uparrow$} & \textbf{ARI $\uparrow$} & \textbf{ACC $\uparrow$} & \textbf{NMI $\uparrow$} & \textbf{ARI $\uparrow$}  & \textbf{ACC $\uparrow$} & \textbf{NMI $\uparrow$} & \textbf{ARI $\uparrow$} \\ 
    \midrule[1.5pt]
        K-means & 34.7 & 16.7 & 25.4 & 38.5 & 17.1 & 30.5 & 57.3 & 29.1 & 57.4 \\
        \midrule[1.2pt]
        Spectral Clustering & 34.2 & 19.5 & 30.2 & 25.9 & 11.8 & 29.5 & 39.7 & 3.5 & 52.0 \\
        DeepWalk \citep{Deepwalk_Perozzi_2014} & 46.7 & 31.8 & 38.1 & 36.2 & 9.7 & 26.7 & 61.9 & 16.7 & 47.1 \\ 
        Louvain \citep{Louvain_Blondel_2008} & 52.4 & 42.7 & 24.0 & 49.9 & 24.7 & 9.2 & 30.4 & 20.0 & 10.3 \\ 
        \midrule[1.2pt]
        GAE [NeurIPS'16] \citep{kipf2016variational} & 61.3 & 44.4 & 38.1 & 48.2 & 22.7 & 19.2 & 63.2 & 24.9 & 24.6 \\
        DGI [ICLR'19] \citep{velickovic2018deepgraphinfomax} & 71.3 & 56.4 & 51.1 & 68.8 & 44.4 & 45.0 & 53.3 & 18.1 & 16.6 \\
        GIC [PAKDD'21] \citep{Mavromatis2021GIC_PAKDD} & 72.5 & 53.7 & 50.8 & 69.6 & 45.3 & 46.5 & 67.3 & 31.9 & \textbf{29.1} \\ 
        DAEGC [IJCAI'19] \citep{DAEGC_IJCAI} & 70.4 & 52.8 & 49.6 & 67.2 & 39.7 & 41.0 & 67.1 & 26.6 & 27.8 \\
        GALA [ICCV'19] \citep{gala_iccv19} & \textbf{74.5} & \underline{57.6} & \underline{53.1} & 69.3 & 44.1 & 44.6 & 69.3 & \underline{32.1} & 32.1 \\
        AGE [KDD'20] \citep{AGE2020} & \underline{73.5} & 57.5 & 50.0 & 69.7 & 44.9 & 34.1 & \textbf{71.1} & 31.6 & 33.4 \\
        DCRN [AAAI'22] \citep{DCRN} & 61.9 & 45.1 & 33.1 & \underline{70.8} & \underline{45.8} & \underline{47.6} & \underline{69.8} & 32.2 & 31.4 \\
        FGC [JMLR'23] \citep{Manoj2023FGC} & 53.8 & 23.2 & 20.5 & 54.2 & 31.1 & 28.2 & 67.1 & 26.6 & 27.8 \\
        \textbf{Q-MAGC (Ours)} & 65.8 & 51.8 & 42.0 & 65.9 & 40.8 & 40.1 & 66.7 & \textbf{32.8} & \underline{27.9} \\
        \textbf{Q-GCN (Ours)} & 71.6 & \textbf{58.3} & \textbf{53.6} & \textbf{71.5} & \textbf{47.0} & \textbf{49.1} & 64.1 & \underline{32.1} & 26.5 \\ 
        \midrule
        SCGC [IEEE TNNLS'23] \citep{SCGC_2023} & \textbf{73.8} & 56.1 & 51.7 & \textbf{71.0} & \underline{45.2} & 46.2 & - & - & - \\
        MVGRL [ICML'20] \citep{mvgrl_icml2020} & \underline{73.2} & \underline{56.2} & 51.9 & \underline{68.1} & 43.2 & 43.4 & \underline{69.3} & \textbf{34.4} & \textbf{32.3} \\
        VGAE [NeurIPS'16] \citep{kipf2016variational} & 64.7 & 43.4 & 37.5 & 51.9 & 24.9 & 23.8 & \textbf{69.6} & 28.6 & \underline{31.7}\\
        ARGA [IJCAI'18] \citep{arva_agrva_ijcai2018p362} & 64.0 & 35.2 & \underline{61.9} & 57.3 & 34.1 & \textbf{54.6} & 59.1 & 23.2 & 29.1 \\
        ARVGA [IJCAI'18] \citep{arva_agrva_ijcai2018p362} & 63.8 & 37.4 & \textbf{62.7} & 54.4 & 24.5 & \underline{52.9} & 58.2 & 20.6 & 22.5 \\ 
        R-VGAE [IEEE TKDE'22] \citep{Mrabah2022RethinkingVGAE} & 71.3 & 49.8 & 48.0 & 44.9 & 19.9 & 12.5 & 69.2 & 30.3 & 30.9\\
        \textbf{Q-VGAE (Ours)} & 72.7 & \textbf{58.6} & 49.6 & 66.1 & \textbf{47.4} & 50.2 & 64.3 & \underline{32.6} & 28.0 \\ 
        \midrule
        VGAECD-OPT [Entropy'20] \citep{vgaecd-opt} & 27.2 & 37.3 & 22.0 & 51.8 & 25.1 & 15.5 & 32.2 & 25.0 & 26.1 \\
        Mod-Aware VGAE [NN'22] \citep{salhagalvan2022modularityawaregae} & 67.1 & 52.4 & 44.8 & 51.8 & 25.1 & 15.5 & - & 30.0 & 29.1 \\
        GMM-VGAE [AAAI'20] \citep{Hui2020_GMM_VGAE} & 71.9 & 53.3 & 48.2 & 67.5 & 40.7 & 42.4 & \underline{71.1} & 29.9 & 33.0 \\
        R-GMM-VGAE [IEEE TKDE'22] \citep{Mrabah2022RethinkingVGAE} & \textbf{76.7} & \underline{57.3} & \textbf{57.9} & \underline{68.9} & \underline{42.0} & \underline{43.9} & \textbf{74.0} & \underline{33.4} & \textbf{37.9} \\
        \textbf{Q-GMM-VGAE (Ours)} & \underline{76.2} & \textbf{58.7} & \underline{56.3} & \textbf{72.7} & \textbf{47.4} & \textbf{48.8} & 69.0 & \textbf{34.8} & \underline{34.0} \\
    \bottomrule[1.5pt]
    \end{tabular}
    \end{adjustbox}
    \caption{\textbf{Comparison of all methods on attributed datasets.} We classify the baselines into three primary groups: the first includes traditional clustering algorithms and GNN-free methods, serving as relevant baselines for our Q-MAGC method. The second category compares GNN-based methods with Q-GCN, and the last, comprising the most performant methods, consists of VGAE-based baselines for Q-VGAE and Q-GMM-VGAE.}
    \label{table:attributed_results}
    \vspace{-\baselineskip}
\end{table*}

\subsection{Non-Attributed Graph Clustering}

For non-attributed graphs, we use a one-hot encoding of the degree vector as features. While this is a basic approach compared to learning-based methods like DeepWalk \citep{perozzi2014deepwalk} and node2vec \citep{grover2016node2vec}, it ensures a fair comparison since other methods also use this feature representation. Our results in Table \ref{table:non_attributed_results} demonstrate that our method achieves competitive or superior performance in terms of NMI, even for non-attributed datasets.

\subsection{Ablation Studies}

\textbf{Comparison of running times}
In Fig \ref{fig:time_comparison}, we compare the running times of our method with other baselines. Our method consistently takes less than half the time across all datasets. Notably, on PubMed (large dataset), state-of-the-art methods GMM-VGAE and R-GMM-VGAE (unmodified) require approximately 60 minutes for clustering, whereas our Q-GMM-VGAE delivers superior performance in under 15 minutes, representing a $75\%$ reduction. Additionally, Q-MAGC runs even faster, completing in just 6 minutes on PubMed, and achieves approximately $90\%$ of the performance.

\begin{figure*}[t!]
    \centering

    \begin{subtable}[c]{0.59\linewidth}
        \centering
        \begin{adjustbox}{width=\linewidth}
        \begin{tabular}{lccc@{\hskip 20pt}ccc@{\hskip 20pt}ccc}
        \toprule[1.5pt]
        & \multicolumn{3}{c}{Brazil} & \multicolumn{3}{c}{Europe} & \multicolumn{3}{c}{USA} \\
        \cmidrule(r){2-4} \cmidrule(r){5-7} \cmidrule(r){8-10}
        \textbf{Method} & \textbf{ACC $\uparrow$} & \textbf{NMI $\uparrow$} & \textbf{ARI $\uparrow$} & \textbf{ACC $\uparrow$} & \textbf{NMI $\uparrow$} & \textbf{ARI $\uparrow$}  & \textbf{ACC $\uparrow$} & \textbf{NMI $\uparrow$} & \textbf{ARI $\uparrow$} \\
        \midrule[1.2pt]
            GAE [NeurIPS'16] & 62.6 & 37.8 & 30.8 & 47.6 & 19.9 & 12.7 & 43.9 & 13.6 & 11.8 \\
            DGI [ICLR'19] & 64.9 & 31.0 & 30.4 & 48.6 & 16.1 & 12.3 & \textbf{52.2} & 22.9 & \textbf{21.7} \\
            GIC [PAKDD'21] & 40.5 & 23.5 & 14.1 & 40.4 & 9.4 & 6.2 & 49.7 & 22.1 & \underline{19.9} \\ 
            DAEGC [AAAI'19] & \underline{71.0} & \textbf{47.4} & 41.2 & \underline{53.6} & 30.9 & 23.3 & 46.4 & \textbf{27.2} & 18.4\\ 
            \textbf{Q-GCN (Ours)} & 51.1 & 31.9 & 23.7 & 45.5 & 30.8 & \underline{25.1} & 43.8 & 19.1 & 14.8 \\ 
            \midrule
            VGAE [NeurIPS'16] & 64.1 & 38.0 & 30.7 & 49.9 & 23.5 & 16.7 & 45.8 & \underline{23.1} & 15.7\\
            \textbf{Q-VGAE (Ours)} & 50.1 & 35.0 & 19.8 & 46.6 & 19.5 & 17.5 & 46.2 & 19.5 & 16.9 \\ 
            \midrule
            GMM-VGAE [AAAI'20] & 70.2 & \underline{46.0} & 41.9 & 53.1 & 31.1 & 24.4 & 48.1 & 21.9 & 13.2 \\
            R-GMM-VGAE [IEEE TKDE'22] & \textbf{73.3} & 45.6 & \textbf{42.5} & \textbf{57.4} & \underline{31.4} & \textbf{25.8} & \underline{50.8} & \underline{23.1} & 15.3 \\
            \textbf{Q-GMM-VGAE (Ours)} & 68.4 & \underline{46.0} & \underline{42.4} & 47.9 & \textbf{32.2} & 23.5 & 46.6 & \underline{23.1} & 13.1 \\
        \bottomrule[1.2pt]
        \end{tabular}
        \end{adjustbox}
        \caption{Comparison of all methods on non-attributed datasets.}
        \label{table:non_attributed_results}
    \end{subtable}
    \hfill
    \begin{subfigure}[c]{0.4\linewidth}
        \includegraphics[width=\linewidth]{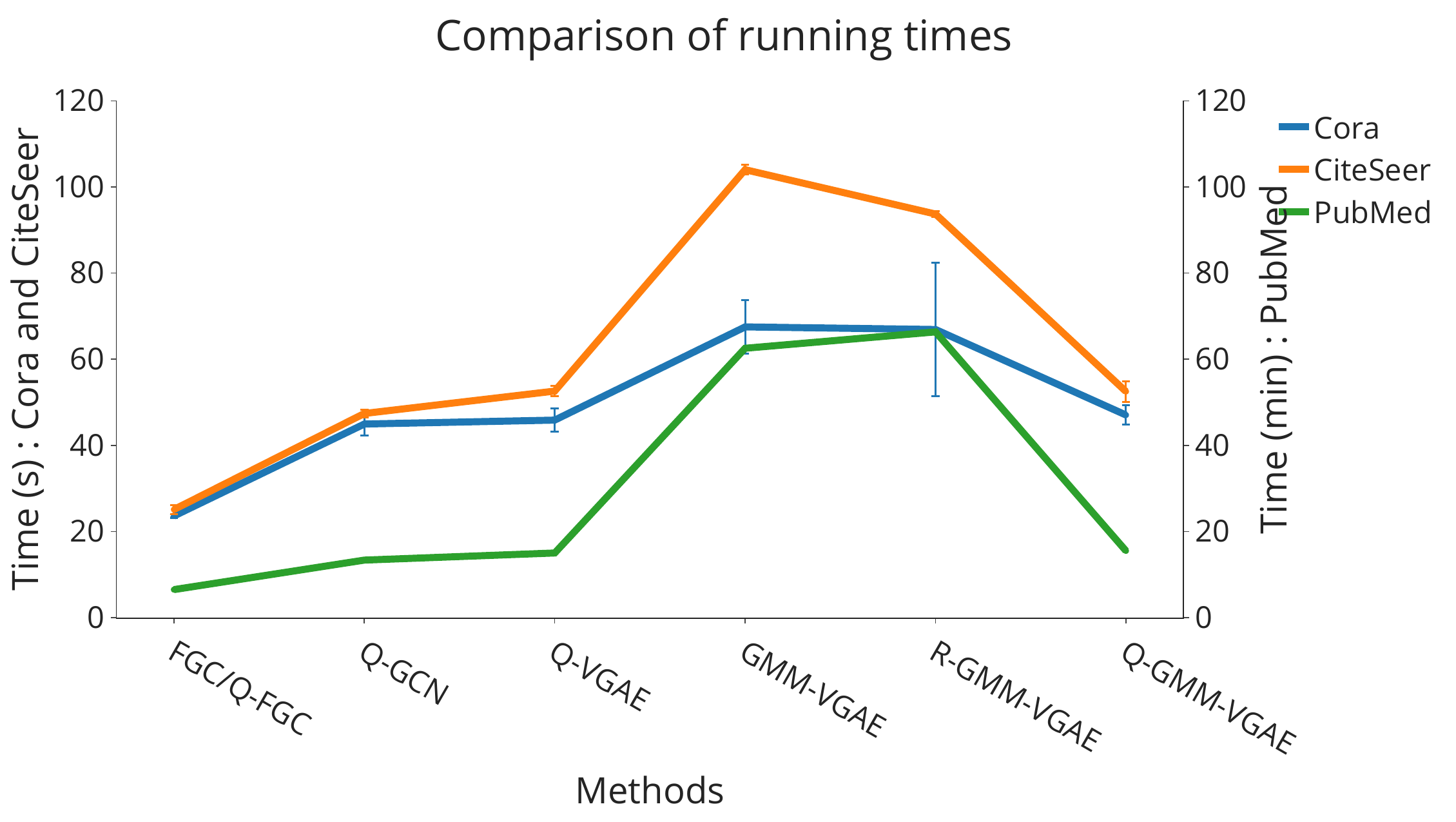}
        \caption{Comparison of running times of methods.
        scale for PubMed (right axis): mins, scale for others (left axis) : secs.}
        \label{fig:time_comparison}
    \end{subfigure}
\end{figure*}

\textbf{Modularity Metric Comparison}
We treat modularity as a metric and measure the gains observed in modularity over other baselines on the Cora, CiteSeer, and PubMed datasets. We report two types of graph-based metrics: modularity \(\mathcal{Q}\) and conductance \(\mathcal{C}\), which do not require labels. Conductance measures the fraction of total edge volume pointing outside the cluster, with \(\mathcal{C}\) being the average conductance across all clusters, where a lower value is preferred.

From Table \ref{table:modularity}, we observe that although DMoN \citep{DMoN_JMLR_2023} achieves the highest modularity, our method attains significantly higher NMI. For CiteSeer, we achieve a $40\%$ improvement in NMI with only an $8\%$ decrease in modularity, positioning us closer to the ground truth. Additionally, our methods outperform their foundational counterparts, with Q-MAGC outperforming FGC, and Q-VGAE outperforming VGAE.
Although modularity is a valuable metric to optimize, the maximum modularity labeling of a graph does not always correspond to the ground truth labeling. Therefore, it is crucial to include the other terms in our formulation. While optimizing modularity helps approach the optimal model parameters (where NMI would be 1), it can deviate slightly. The additional terms in our formulation correct this trajectory, ensuring more accurate results.

\begin{figure}[t!]
    \centering
    \includegraphics[width=\linewidth]{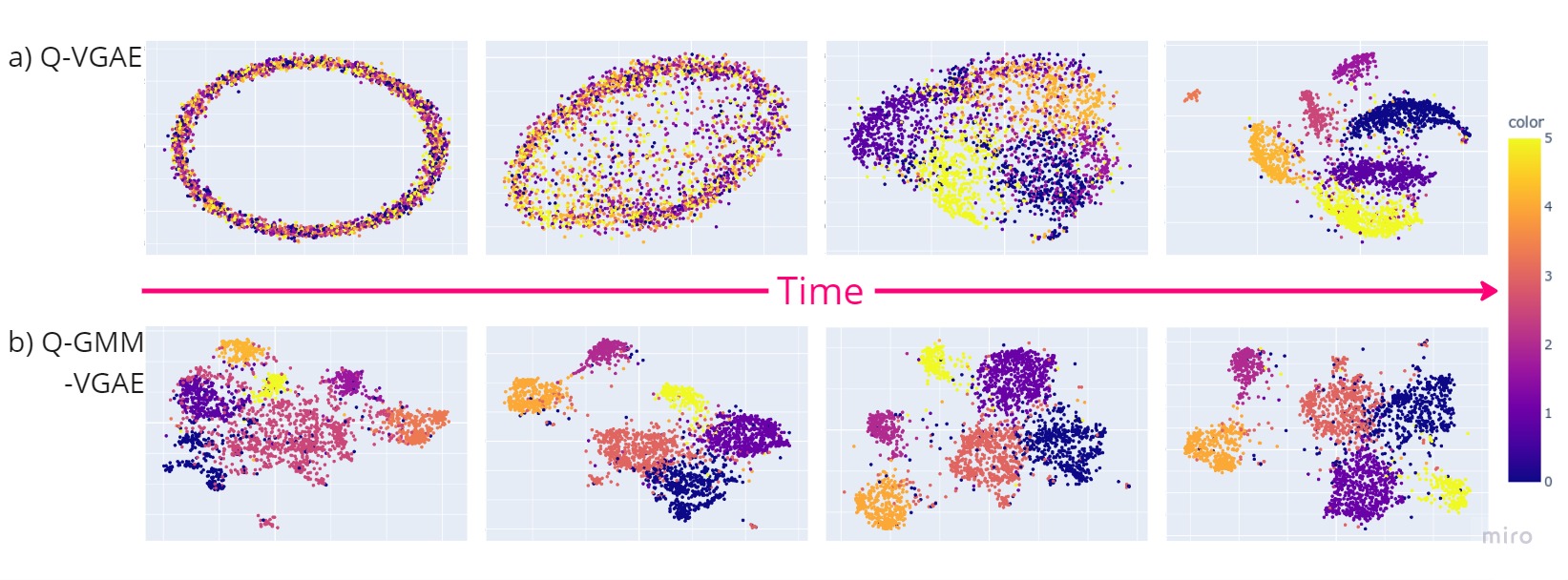}
    \caption{Evolution of the latent space of a) Q-VGAE and b) Q-GMM-VGAE over time for Cora. Colors represent cluster assignments.}
    \label{fig:cora_latent}
\end{figure}

{
\setlength\textfloatsep{0pt}
\setlength\floatsep{-10pt}
\setlength\intextsep{0pt}

\begin{figure*}
    \centering
     \begin{subtable}[c]{0.49\linewidth}
        \centering
        \begin{adjustbox}{width=\linewidth}
        \begin{tabular}{lrrrrrrrrr}
        \toprule
        ~ & \multicolumn{3}{c}{Cora} & \multicolumn{3}{c}{CiteSeer} & \multicolumn{3}{c}{PubMed}  \\ 
        \cmidrule(r){2-4} \cmidrule(r){5-7} \cmidrule(r){8-10}
        ~ & $\mathcal{C} \downarrow$ & $\mathcal{Q} \uparrow$ & NMI $\uparrow$ & $\mathcal{C} \downarrow$ & $\mathcal{Q} \uparrow$ & NMI $\uparrow$ & $\mathcal{C} \downarrow$ & $\mathcal{Q} \uparrow$ & NMI $\uparrow$ \\ 
        \midrule
            DMoN & 12.2 & \textbf{76.5} & 48.8 & 5.1 & \textbf{79.3} & 33.7 & 17.7 & \textbf{65.4} & 29.8  \\ 
            FGC & 58.4 & 25 & 23.1 & 41.6 & 41.1 & 31 & 21.6 & 44.1 & 20.5  \\ 
            Q-MAGC & 13.3 & 72.5 & 51.7 & 16.8 & 64.9 & 40.16 & 26 & 40.3 & 28.1  \\ 
            Q-GCN & 13.6 & 73.3 & 58.3 & 5.8 & 74.5 & 46.7 & \textbf{8.27} & 55 & 31.5  \\ 
            VGAE & 17.6 & 60.8 & 38.1 & 12.8 & 55.8 & 21 & 13.5 & 45.8 & 26.9  \\ 
            Q-VGAE & \textbf{9.5} & 71.5 & \textbf{58.4} & \textbf{4.6 }& 72.4 & \textbf{47.3} & 9.4 & 52.12 & \textbf{31.8} \\ 
        \bottomrule
        \end{tabular}
        \end{adjustbox}
        \caption{Comparison of modularity and conductance at the best NMI with DMoN. Note that DMoN is optimizing only modularity, whereas we are optimizing other important terms as well, as mentioned in Eqn \ref{optimization_objective}, and thus gain a lot on NMI by giving up a small amount of modularity, making us closer to the ground truth. }
        \label{table:modularity}
    \end{subtable}
    \hfill
    \begin{subfigure}[c]{0.49\linewidth}
        \centering
        \includegraphics[width=0.8\linewidth]{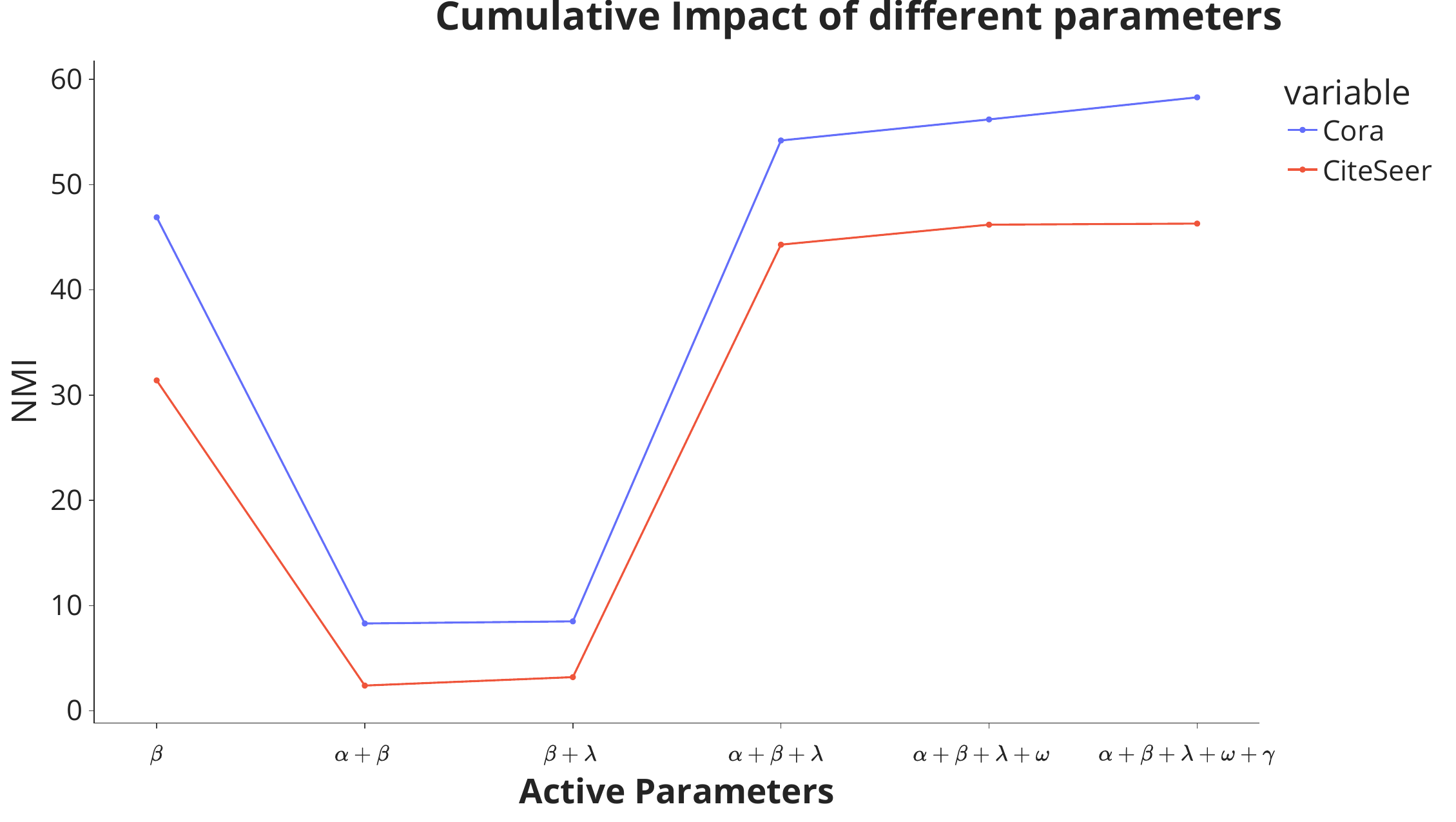}
        \caption{Impact of active parameters on clustering performance. 
        All terms in the loss \eqref{optimization_objective} are represented by their parameters, for example the modularity term is represented by $\beta$. Additionally, $\omega$ represents the non-parameterized term.}
        \label{fig:loss_term_impact}
    \end{subfigure}
\vspace{-\baselineskip}
\end{figure*}
}

\textbf{Importance of and Evolution of different loss terms}
We analyze the evolution of the different loss terms during training, and also try to measure the impact of each term separately by removing terms from the loss one by one, as shown in Appendix \ref{supp:loss_evol}.
Also, we found that $||CX_C - X||_F^2$ is the most sensitive to change in its weight $\alpha$, followed by the terms related to $\gamma$, $\beta$ and then $\lambda$. This makes sense because if that constraint(relaxation) is not being met, then $C$ would have errors. \
Even though some of the terms do the heavy lifting, the other regularization terms do contribute to performance and more importantly, change the nature of $C$ : The smoothness term corresponds to smoothness of signals in the graph being transferred to the coarsened graph which encourages local "patches"/groups of $C$ to belong to the same cluster. The term $\gamma$ ensures that the coarsened graph is connected - i.e. preserving inter-cluster relations, which simple contrastive methods destroy; this affects $C$ by making it so that $\Theta_C$ has minimal multiplicity of 0-eigenvalues.

\textbf{Visualization of the latent space}
In Figure \ref{fig:cora_latent}, we visualize how the latent space of the Q-VGAE and Q-GMM-VGAE changes over time for the Cora dataset. Plots for the rest of the datasets can be found in the Appendix \ref{supp:latent_plots}.
We use UMAP (Uniform Manifold Approximation and Projection) \citep{UMAPMcInnes2018} for dimensionality reduction. 



\section{Discussion}
\label{discussion}
\textbf{Performance of.}
Our methods consistently outperform their counterparts. Q-MAGC demonstrates significant superiority over traditional algorithms and GNN-free approaches. These can be attributed to the inclusion of modularity in our objective, as well as the smoothness and connectedness terms, making Q-MAGC more effective than previous methods.

To compete with GNN-based methods, we leverage their powerful message-passing mechanisms. Q-GCN surpasses other GCN/simple GNN-based methods, such as AGE and DCRN, not only because it optimizes modularity but also due to its focus on smoothness and inter-cluster relationships in the coarsened graph. Similarly, Q-VGAE achieves better results than other VGAE/complex GNN-based methods for the same reasons.

\textbf{Efficiency.}
Our methods are efficient in both complexity and implementation, as demonstrated in the Complexity Analysis \ref{alg:q-magc} and the Ablation - Comparison of running times.

\textbf{Limitations.} If the modularity of a graph, calculated from the ground truth labeling, is low, our method may not perform optimally and can be slower to converge since maximizing modularity in such cases is not ideal. However, our method still manages to match or surpass state-of-the-art methods on the Airports dataset, where all graphs exhibit low modularity based on ground truth labels. This is due to the other terms in our weighted loss function, which become the primary contributors to performance.

\bibliography{main}
\bibliographystyle{tmlr}

\newpage
\appendix

\textbf{\Large Appendix}

\startcontents[sections]
\printcontents[sections]{l}{1}{\setcounter{tocdepth}{2}}

\newpage
\section{Convexity of terms in the optimization objective \ref{optimization_objective}}
\label{supp:convexity_proof}
When $X_C$ is kept constant, $\mathcal{L}_{MAGC}$ gets reduced to: 
\begin{align}
    \underset{C}{\min}\;  f(C) = \ \textrm{tr}(X_C^TC^T\Theta CX_C) + & \frac{\alpha}{2}\norm{CX_C -X}_F^2 - \frac{\beta}{2e}\textrm{tr}(C^TBC)\\
    &- \gamma \log\det(C^T\Theta C + J) + \frac{\lambda}{2}\norm{C^T}_{1,2}^2\nonumber\\
    \text{subject}\;\;\text{to}\ \mathcal{C} \in \mathcal{S}_c \ref{coarsened_features_equation} \nonumber\; \text{where, } J = \frac{1}{k} \textbf{1}_{k\times k} \\
    \label{sub_problem_C}
\end{align}
The term $\textrm{tr}(X_C^TC^T\Theta CX_C)$ is convex function in $C$. This result can be derived easily using Cholesky Decomposition on the positive semi-definite matrix $\Theta$ (i.e. $\Theta = L^TL$):
\begin{align}
    \textrm{tr}(X_C^TC^T\Theta CX_C) = \textrm{tr}(X_C^TC^TL^TLCX_C) = \textrm{tr}(({LCX_C})^TLCX_C) = \norm{LCX_C}_F^2
\end{align}
Frobenius norm is a convex function, and the simplified expression is linear in C. Hence we can deduce that the tr(.) term is convex in C.
The terms $\norm{CX_C -X}_F^2$ and $\norm{C^T}_{1,2}^2$ are convex because Frobenius norm and $l_{1,2}$ norm are convex in C. 


For proving the convexity of $-\log\det(C^T\Theta C + J)$ we restrict function to a line. We define a function $g$:
\begin{align}
    g(t) = f(z + tu) \text{where, } t \in dom(g), z \in dom(f), u \in \mathbb{R}^n.
\end{align}
A function $f: \mathbb{R}^n \rightarrow \mathbb{R}$ is convex if $g: \mathbb{R} \rightarrow \mathbb{R}$ is convex. 

The graph Laplacian matrix of the coarsened graph ($\Theta_c$) is symmetric and positive semi-definite having a rank of k-1. To convert $\Theta_c$ to positive definite matrix, we add a rank 1 matrix $J = \frac{1}{k} \textbf{1}_{k\times k}$. ($\Theta_c + J = L^TL$)

\begin{align}
    f(L) = -\log\det(C^T\Theta C + J) = -\log\det(L^TL) 
\end{align}
Now substituting L = Z + tU in the above equation.
\begin{align}
    g(t) &= -\log\det((Z + tU)^T (Z + tU)) \\
    &= -\log\det(Z^TZ + t(Z^TU + U^TZ) t^2U^TU) \\
    &= -\log\det(Z^T(I + t(UZ^{-1} + {(UZ^{-1})}^T) + t^2{(Z^{-1})}^TU^TUZ^{-1})Z) \\
    \text{substituting } \ P &= VZ^{-1} \\
    &= -(\log \det(Z^TZ) + \log\det(I + t(P+P^T) + t^2P^TP)) \\
    \text{Eigenvalue decomposition of} P &= Q\Lambda Q^T \text{and} QQ^T = I \\
    &= -(\log \det(Z^TZ) + \log\det (QQ^T + 2tQ\Lambda Q^T + t^2Q\Lambda^2Q^T)) \\
    &= -(\log \det(Z^TZ) + \log\det(Q(I + 2t\Lambda + t^2\Lambda^2)Q^T)) \\
    &= -\log\det(Z^TZ) -\sum_{i=1}^{n}\log(1 + 2t\lambda_i + t^2\lambda^2) 
\end{align}
Finding double derivative of $g(t)$:
\begin{align}
    g"(t) = \sum_{i=1}^{n} \frac{2\lambda_i^2(1 + t\lambda_i)^2}{(1 + 2t\lambda_i + t^2\lambda_i^2)^2}
\end{align}
Since $g"(t) \geq 0 \forall \ t\in\mathbb{R}$, $g(t)$ is a convex function in $t$. This implies $f(L)$ is convex in $L$. We know that, $C^T\Theta C + J = L^TL$ so,
\begin{align}
    L = \Theta^{\frac{1}{2}}C + \frac{1}{\sqrt{kp}}\textbf{1}_{p \times k} 
\end{align}
Since $L$ is linear in $C$ and $f(L)$ is convex in $L$, $-\log\det(C^T\Theta C + J)$ is convex in C.

\section{Optimal Solution of Optimization Objective in Equation \ref{problem in C}}
\label{supp:optimal_soln}
We first show that the function $f(C)$ is $L-Lipschitz$ continuous gradient function with $L = \max(L_1,L_2,L_3,L_4,L_5)$, where $L_1,L_2,L_3,L_4, and L_5$ are the Lipschitz constants of $\textrm{tr}(X_C^TC^T\Theta CX_C), \frac{\alpha}{2}\norm{CX_C -X}_F^2, -\frac{\beta}{2e}\textrm{tr}(C^TBC), - \gamma \log\det(C^T\Theta C + J), and \frac{\lambda}{2}\norm{C^T}_{1,2}^2$. 

For the $\textrm{tr}(X_C^TC^T\Theta CX_C)$ term, we apply triangle inequality and employ the property of the norm of the trace operator: $||tr|| = \underset{M \neq 0}{sup} \frac{|tr(M)|}{||M||_F}$. 
\begin{align}
    &|tr(X_C^TC_1^T\Theta C_1X_C) - tr(X_C^TC_2^T\Theta C_2X_C)| \\ 
    &= |tr(X_C^TC_1^T\Theta C_1X_C) - tr(X_C^TC_2^T\Theta C_1X_C) + tr(X_C^TC_2^T\Theta C_1X_C) -tr(X_C^TC_2^T\Theta C_2X_C)| \\
    &\leq |tr(X_C^TC_1^T\Theta C_1X_C) - tr(X_C^TC_2^T\Theta C_1X_C)| + |tr(X_C^TC_2^T\Theta C_1X_C) -tr(X_C^TC_2^T\Theta C_2X_C)| \\
    &\leq ||tr|| ||X_C^T(C_1-C_2)^T\Theta C_1X_C||_F + ||tr|| ||X_C^TC_2^T\Theta (C_1-C_2)X_C||_F \\
    &\leq ||tr|| ||X_C||_F ||\Theta||||C_1-C_2||_F (||C_1||_F + ||C_2||_F) \ \ \text{(Frobenius Norm Property)} \\
    &\leq 2 \sqrt{p} ||tr|| ||X_C||_F ||\Theta||||C_1-C_2||_F \ \ (||C_1||_F = ||C_2||_F = \sqrt{p}) \\
    &\leq L_1 ||C_1-C_2||_F
\end{align}

The second term is $\frac{\alpha}{2}\norm{CX_C -X}_F^2$ can be written as:
\begin{align}
    &\frac{\alpha}{2} tr((CX_C -X)^T(CX_C -X)) \\
    &= \frac{\alpha}{2} tr(X_C^TC^TCX_C - X^TCX_C + X^TX - X_C^TC^TX) \\
    &= \frac{\alpha}{2} (tr(X_C^TC^TCX_C) - tr(X^TCX_C) + tr(X^TX) - tr(X_C^TC^TX))
\end{align}
All the terms except $tr(X^TX)$ (constant with respect to C) in obtained in the expression will follow similar proofs to $\textrm{tr}(X_C^TC^T\Theta CX_C)$. 

Next we consider the modularity term:
\begin{align}
    &|tr(C_1^TBC_1) - tr(C_2^TBC_2)| \\
    &= |tr(C_1^TBC_1) - tr(C_2^TBC_1) + tr(C_2^TBC_1) - tr(C_2^TBC_2)| \\
    &\leq  |tr(C_1^TBC_1) - tr(C_2^TBC_1)| + |tr(C_2^TBC_1) - tr(C_2^TBC_2)| \\
    &\leq ||tr||||(C_1-C_2)^TBC_1||_F + ||tr||||(C_1-C_2)^TBC_2||_F \\
    &\leq ||tr||||B||||C_1-C_2||_F (||C_1||_F + ||C_2||_F) \ \ \text{(Frobenius Norm Property)} \\
    &\leq L_3 ||C_1-C_2||_F   
\end{align}
The Lipschitz constant for $- \gamma \log\det(C^T\Theta C + J)$ is linked to the smallest non-zero eigenvalue of the coarsened Laplacian matrix ($\Theta_c$) and is bounded by $\frac{\delta}{(k-1)^2}$ \citep{log_det_bound_sandeep_kumar}, where $\delta$ is the minimum non-zero weight of $G_c$.

$tr(\textbf{1}^TC^TC\textbf{1})$ can be proved to be $L_5-Lipschitz$ like the modularity and Dirichlet energy (smoothness) terms. This concludes the proof.

The majorized problem for L-Lipschitz and differentiable functions can now be applied. 
The Lagrangian of the majorized problem, \ref{majorized_objective} is:

\begin{align}
    \mathcal{L}(C, X_C, \mu) = \frac{1}{2}C^TC - C^TA - \mu_1^TC + \mu_2^T\Big[\norm{C_1^T}_2^2 -1,\cdots, \norm{C_i^T}_2^2 -1,\cdots, \norm{C_p^T}_2^2 -1  \Big]^T
\end{align}
where $\mu = \mu_1||\mu_2$ are the dual variables and $A =  \Big(C - \frac{1}{L}\nabla f(C) \Big)^{+}$

The corresponding KKT conditions (w.r.t $C$) are:
\begin{align}
    C - A - \mu_1 + 2[{\mu_2}_oC_0^T, \cdots, {\mu_2}_iC_i^T, \cdots, {\mu_2}_pC_p^T] &= 0\\
    \mu_2^T\Big[\norm{C_1^T}_2^2 -1,\cdots, \norm{C_i^T}_2^2 -1,\cdots, \norm{C_p^T}_2^2 -1,  \Big]^T &= 0 \\
    \mu_1^TC &= 0 \\
    \mu_1 & \geq 0 \\
    \mu_2 & \geq 0
    C &\geq 0 \\
    \norm{[C^T]_i}_2^2 &\leq 1\ \forall i
\end{align}

The optimal solution to these KKT conditions is:
\begin{align}
    C = \frac{(A)^{+}}{\sum_i\norm{[A^T]_i}}_2
\end{align}

\section{Proof of Theorem \ref{theorem:convergence} (Convergence)}
\label{supp:convergence}
In this section, we prove that the sequence \( \{C^{t+1}, X_C^{t+1}\} \) generated by Algorithm \ref{alg:q-magc} converges to the set of Karush–Kuhn–Tucker (KKT) optimality points for Problem \ref{optimization_objective}.

The Lagrangian of Problem \ref{optimization_objective} comes out to be:
\begin{align}
    \mathcal{L}(C, X_C, \mathbf{\mu}) &= \textrm{tr}(X_C^TC^T\Theta CX_C) +  \frac{\alpha}{2}\norm{CX_C -X}_F^2 - \frac{\beta}{2e}\textrm{tr}(C^TBC)\\
    &\ \ - \gamma \log\det(C^T\Theta C + J) + \frac{\lambda}{2}\norm{C^T}_{1,2}^2 - \mathbf{\mu}_1^TC + \sum_i {\mathbf{\mu}_2}_i\Big[\norm{C^T_i}_2^2 - 1\Big]
\end{align}
where $\mu = \mu_1||\mu_2$ are the dual variables.

\textbf{w.r.t. $C$}, the KKT conditions are
\begin{align}
    2\Theta CX_CX_C^T + \alpha(CX_C - X)X_C^T - \frac{\beta}{e}BC - 2\gamma\Theta C(C^T\Theta C + J)^{-1}&\\ + \lambda C\mathbf{1}_{k\times k} - \mu_1 + 2[{\mu_2}_oC_0^T, \cdots, {\mu_2}_iC_i^T, \cdots, {\mu_2}_pC_p^T] &= 0 \nonumber\\
    \mu_2^T\Big[\norm{C_1^T}_2^2 -1,\cdots, \norm{C_i^T}_2^2 -1,\cdots, \norm{C_p^T}_2^2 -1 \Big]^T &= 0 \\
    \mu_1^TC &= 0 \\
    \mu_1 &\geq 0 \\
    \mu_2 & \geq 0 \\
    C &\geq 0 \\
    \norm{[C^T]_i}_2^2 &\leq 1\ \forall i
\end{align}

Now, $C^\infty \equiv \underset{t \rightarrow \infty}{\lim}{C^t}$ is found from Equation \ref{eqn:C_Update} as:

\begin{align}
    C^\infty &= C^\infty + \frac{1}{L} \Bigg( 
    2\Theta C^\infty X_C^\infty X_C^\infty + \alpha(C^\infty X_C - X)X_C^\infty - \frac{\beta}{e}BC^\infty \\&\ \ \ \ - 2\gamma\Theta C^\infty({C^\infty}^T\Theta C^\infty + J)^{-1} + \lambda C^\infty\mathbf{1}_{k\times k} \Bigg) \nonumber \\
    0 &= 2\Theta C^\infty X_C^\infty X_C^\infty + \alpha(C^\infty X_C - X)X_C^\infty - \frac{\beta}{e}BC^\infty \\&\ \ \ \ - 2\gamma\Theta C^\infty({C^\infty}^T\Theta C^\infty + J)^{-1} + \lambda C^\infty\mathbf{1}_{k\times k} \nonumber
\end{align}

So, for $\mu = 0$, $C^\infty$ satisfies the KKT conditions.

\textbf{w.r.t. $X_C$}, the KKT conditions are:

\begin{align}
    2C^T\Theta CX_C + \alpha C^T(CX_C - X) = 0
\end{align}

So, $X^\infty \equiv \underset{t \rightarrow \infty}{\lim}{X^t}$ found from Equation \ref{eqn:X_Update} will satisfy this as that equation is just a rearrangement of the KKT condition.

\section{Proof (continued) of Theorem \ref{theorem2:consistency} (conditions for consistency)}
\label{supp:proof_consistency_conditions}
As previously defined in Theorem \ref{theorem2:consistency} $O = C^T A C$ and
$$
\frac{1}{\mu_p} \mathbb{E}[O | C, t] = H(S) \qquad
$$

We need to find "population version" of the loss function in terms of $H(S)$.

\begin{alignat}{2}
\label{eqn:supp_expect_CTAC}
\mathbb{E}[O | C, t] &= \mathbb{E}[C^TAC | C, t] = \mu_p H(S) \\
\mathbb{E}[C^TDC | C, t] &= \mathbb{E}[C^T \text{diag}(\sum_{j=1}^k A_{ij}) C | C, t] 
&&= \mathbb{E}[\text{diag}(\sum_{j=1}^k O_{ij}) | C, t] \\
\label{eqn:supp_expect_CTDC}
&= \mu_p \text{diag}(\sum_{j=1}^{k} H_{ij}) &&
\end{alignat}

So, for $\text{tr}(X_C^TC^T\Theta CX_C)$

\begin{alignat}{2}
& \mathbb{E}[\text{tr}(X_C^TC^T\Theta CX_C) | C, t] & \nonumber
&= \text{tr}(\mathbb{E}[X_C^TC^T\Theta CX_C | C, t]) \\
&= \text{tr}(X_C^T\mathbb{E}[C^T\Theta C | C, t]X_C) & \nonumber
&= \text{tr}(X_C^T\mathbb{E}[C^TDC - C^TAC | C, t]X_C) \\
&= \text{tr}(X_C^T [\mu_p \text{diag}(\sum_{j=1}^{k} H_{ij}) && - \mu_p H(S)] X_C) \nonumber \\
&= \mu_p \text{tr}(X_C^T [\text{diag}(\sum_{j=1}^{k} H_{ij}) && - H(S)] X_C)
\end{alignat}

Next, we have $\frac{1}{2e}\text{tr}(C^TBC)$, which has already been solved in the paper \cite{zhao2012_consistency} in their Appendix (Page 25 of the full document).
$$
\frac{1}{2e} \mathbb{E}[\text{tr}(C^TBC)] = \sum_k \bigg(\frac{H_{kk}}{\Tilde{P_0}} - \bigg(\frac{H_k}{\Tilde{P_0}}\bigg)^2 \bigg)
$$
For $\log\det(C^T\Theta C + J)$, where $J = \frac{1}{k}\mathbb{1}_{k\times k}$, 

We can write $\log(\det (C^T\Theta C + J)) = \text{tr}(\log(C^T\Theta C + J))$ since, $\det(A) = e^{\text{tr}(\log(A))}$.

\begin{alignat}{2}
Z = C^T\Theta C + J &= V\Lambda V^{-1}\\
\text{So,}\\
\text{tr}(\log (V\Lambda V^{-1})) &= \text{tr}(V\log (\Lambda)V^{-1})\\
\log \Lambda &= \log (bI) + \log\bigg( I + \frac{\Lambda}{b} - I \bigg)
\label{eqn:supp_expect_logdet}
\end{alignat}


Using the first-order Taylor expansion of $\log (I + X) = X$, we need to choose $b$ such that
\begin{equation}
l = \bigg\| \frac{\Lambda}{b} - I \bigg\|_F < 1
\label{eqn:supp_l_condition}
\end{equation}
And for the expansion to be a good approximation, we need $l \to 0$. We will enforce this later in \eqref{supp:eqn_logdet_goodapprox}.


\begin{alignat}{2}
& \text{tr}(V\bigg(\log (bI) + \log\bigg( I + \frac{\Lambda}{b} - I \bigg)\bigg)V^{-1})\\
=&\ \text{tr}(V\bigg(\log (bI) + \bigg( I + \frac{\Lambda}{b} - I \bigg)\bigg)V^{-1})\\
=&\ \text{tr}(V\log(bI)V^{-1}) + \frac{1}{b}\text{tr}(V\Lambda V^{-1}) - \text{tr}(I)\\
=&\ \text{tr}(\log(bI)) + \frac{1}{b}\text{tr}(Z) - \text{tr}(I)\\
=&\ k\log(b) + \frac{1}{b}\text{tr}(Z) - k
\end{alignat}

Finding the expectation of from \eqref{eqn:supp_expect_logdet},

\begin{align}
&= k\log(b) + \mathbb{E}\bigg[\frac{1}{b}\text{tr}(Z)\bigg|\ c,t \bigg] - k\\
&= k\log(b) + \frac{1}{b}\text{tr}(\mathbb{E}[Z|\ c,t ]) - k\\
&= k\log(b) + \frac{1}{b}\text{tr}(\mathbb{E}[C^T\Theta C + J|\ c,t ]) - k\\
&\text{Using $\Theta = D - A$, \eqref{eqn:supp_expect_CTAC} and \eqref{eqn:supp_expect_CTDC}}, \nonumber \\
&= k\log(b) + \frac{1}{b}\text{tr}\bigg(\mu_p \text{diag}(\sum_{j=1}^{k} H_{ij}) - \mu_p H(S) \bigg) - k
\end{align}

which is a linear function in $H(S)$.




For the approximation to be good, we can now simplify $l$ as defined in \eqref{eqn:supp_l_condition}:

\begin{align}
\label{supp:eqn_logdet_goodapprox}
l &= \bigg\| \frac{\Lambda}{b} - I \bigg\|_F \\
l &= 
\begin{Vmatrix}
\frac{1}{b}\Lambda_{11} - 1 & \frac{1}{b}\Lambda_{12} & \cdots& \frac{1}{b}\Lambda_{1k} \\
\frac{1}{b}\Lambda_{21} & \frac{1}{b}\Lambda_{22} - 1 & \cdots& \frac{1}{b}\Lambda_{2k} \\
\vdots & \vdots & \ddots & \vdots \\
\frac{1}{b}\Lambda_{k1} & \frac{1}{b}\Lambda_{k2} & \cdots& \frac{1}{b}\Lambda_{kk} - 1 \\
\end{Vmatrix}_F
\end{align}

Writing out this norm, we get a quadratic expression in $\frac{1}{b}$:

\begin{align}
l &= (\sum_{i=1}^k\sum_{j=1}^k \lambda_{ij}^2)\frac{1}{b^2} - 2(\sum_{u=1}^k \lambda_{ii})\frac{1}{b} + k \\
\text{or concisely}, 
l &= \|\Lambda\|_F^2 \frac{1}{b^2} - 2\text{tr}(\Lambda)\frac{1}{b} + k\\
\text{Since, }\Lambda&\text{ is a diagonal matrix, }\\ 
\text{tr}(\Lambda) &= \sum_i \lambda_i = \text{tr}(Z)\\
\text{Also, since }&\Lambda^2\text{ is the eigenvalue matrix for }Z^2,\\
\|\Lambda\|_F^2 &= \sum_i\lambda_i^2 = \text{tr}(Z^2)\\
l &= \text{tr}(Z^2) \frac{1}{b^2} - 2\text{tr}(Z)\frac{1}{b} + k
\end{align}


Since, $l<1 \implies l-1 < 0 \implies l-1 = 0$ has 2 real roots.
Using simple quadratic analysis (in $\frac{1}{b}$), the discriminant $\Delta$ should be positive.

\begin{alignat}{2}
\Delta = 4 \text{tr}(Z)^2 \ - 4(k-1)\text{tr}(Z^2) &> 0 \\
\frac{\text{tr}(Z)^2}{\text{tr}(Z^2)} &> k-1
\end{alignat}

\begin{alignat}{2}
\text{The minimum}& \text{ value of } l \coloneqq l_{min}\text{ occurs at }\\
b &= \frac{\text{tr}(Z^2)}{\text{tr}(Z)} = \text{tr}(Z) \cdot \frac{\text{tr}(Z^2)}{\text{tr}(Z)^2}\\
l_{min} &< 1 \text{ will exist when }\frac{\text{tr}(Z)^2}{\text{tr}(Z^2) } > k-1\\
\text{which holds for }& b < \frac{\text{tr}(Z)}{k-1} \\
\text{So we can always choose }b &< \frac{2k-1}{k-1} \\
\text{ since }\min \text{tr}(Z) &= 2k-1 [\text{proved in \ref{eqn:supp_min_trZ}}]
\end{alignat}







\begin{align}
\label{eqn:supp_min_trZ}
\min \text{tr}(C^T\Theta C + J) &= \min \text{tr}(C^T\Theta C) + \text{tr}(J) \\
&= \min \text{tr}(C^TDC) - \text{tr}(C^TAC) + 1 \\
&= 2e - 2(e - (k - 1)) + 1 \\
&= 2k - 1
\end{align}

Additionally, define \(\Tilde{\pi}_a = \sum_u x_u \Pi_{au}\) with \(\sum_a \Tilde{\pi}_a = 1\), since \(\mathbb{E}[t_i] = 1\).

\subsection{Required Condition a): Lipschitz Continuity}
\label{supp:proof_consistency_conditions_1}
Condition 1: We need to show that $|F(S_1) - F(S_2)| \leq \alpha \|S_1 - S_2\|$ (Lipschitz)

\begin{align}
|F(S_1) - F(S_2)| &\leq |f_1(S_1) - f_1(S_2)| + |f_2(S_1) - f_2(S_2)| + |f_3(S_1) - f_3(S_2)|
\end{align}

Let's first find $\|H(S_1) - H(S_2) \|_F$

\begin{align}
H(S) &= (Sx)P(Sx)^T \\
\| H(S_1) - H(S_2) \|_F &= \| (S_1x)P(S_1x)^T - (S_2x)P(S_2x)^T \|_F \\
&= \| (S_1x)P((S_1 - S_2)x)^T + ((S_1 - S_2)x)P(S_2x)^T \|_F
\end{align}

Next, we see $\| \text{diag}(\sum_{j=1}^k H(S_1)_{ij} ) - \text{diag}(\sum_{j=1}^k H(S_2)_{ij} ) \|_F = \| \text{diag}(\sum_{j=1}^k (H(S_1)_{ij} - H(S_2)_{ij}) ) \|_F $


For the first term, $|f_1(H(S_1)) - f_1(H(S_2))|$, define $H(S)_i = \sum_{j=1}^{k} H(S)_{ij}$


\begin{alignat}{1}
&= |\mu_p \text{tr}(X_C^T [H(S_1) - \text{diag}(\sum_{j=1}^{k} H(S_1)_{ij})] X_C) 
- \mu_p \text{tr}(X_C^T [H(S_2) - \text{diag}(\sum_{j=1}^{k} H(S_2)_{ij})] X_C) | 
\\
&= |\mu_p \bigg( \text{tr}(X_C^T [H(S_1) - H(S_2)] X_C) 
-  \text{tr}(X_C^T \text{diag}([H(S_1)_{i} - H(S_2)_{i}])X_C) \bigg) | 
\\
&\leq |\mu_p| \bigg( \bigg|\text{tr}(X_C^T [H(S_1) - H(S_2)] X_C) \bigg|
- \bigg|\text{tr}(X_C^T \text{diag}([H(S_1)_{i} - H(S_2)_{i}])X_C) \bigg| \bigg)
\\
&\leq |\mu_p| \bigg( \|\text{tr}\| \bigg\|X_C^T [H(S_1) - H(S_2)] X_C \bigg\|_F
- \|\text{tr}\|\bigg\|X_C^T \text{diag}([H(S_1)_{i} - H(S_2)_{i}])X_C \bigg\|_F \bigg)
\\
&\leq |\mu_p| \bigg( \|\text{tr}\| \|X_C\|^2 \bigg\|H(S_1) - H(S_2) \bigg\|_F
- \|\text{tr}\| \|X_C\|^2 \bigg\| \text{diag}([H(S_1)_{i} - H(S_2)_{i}]) \bigg\|_F \bigg)
\end{alignat}

Taking the first sub-term,

\begin{align}
&|\mu_p|\ \|\text{tr}\|\ \|X_C\|^2\ \bigg\|H(S_1) - H(S_2) \bigg\|_F \\
&= \|\text{tr}\|\ \|X_C\|^2\ \bigg\|  (S_1x)P((S_1 - S_2)x)^T + ((S_1 - S_2)x)P(S_2x)^T \bigg\|_F \\
&\leq \|\text{tr}\|\ \|X_C\|^2\ \|P\|_F\ (\|S_1x\|_F + \|S_2x\|_F)\ \|(S_1 - S_2)x\|_F \\
&\leq \|\text{tr}\|\ \|X_C\|^2\ \|P\|_F\ (\|S_1x\|_F + \|S_2x\|_F)\ \|t\|_F\ \|(S_1 - S_2)\|_F \\
&= \alpha_1 \|(S_1 - S_2)\|_F
\end{align}

For the second sub-term,
define $\tilde{S}_{ka} = \sum_u x_u S_{kau} = (Sx)_{ka}$

\begin{align}
H(S)_i = \sum_{j=1}^{k} H(S)_{ij} &= \sum_{as} \tilde{\pi}_sP_{as}\tilde{S}_{ia} = \sum_{asu} \tilde{\pi}_s P_{as} x_u S_{iau} \\
|H(S_1)_{i} - H(S_2)_{i}| &= | \sum_{asu} \tilde{\pi}_s P_{as} x_u {S_1}_{iau} - \sum_{asu} \tilde{\pi}_s P_{as} x_u {S_2}_{iau} |\\
&= |\sum_{asu} \tilde{\pi}_s P_{as} x_u ({S_1 - S_2})_{iau}| \\
&\leq \sum_{asu} |\tilde{\pi}_s P_{as} x_u ({S_1 - S_2})_{iau}| \\
&\leq \sum_{asu} |\tilde{\pi}_s P_{as} x_u| \cdot |({S_1 - S_2})_{iau}| \\
&\leq \sum_{asu} |\tilde{\pi}_s P_{as} x_u| \cdot \sum_{au}|({S_1 - S_2})_{iau}| \\
&= \alpha_1^\prime \sum_{au}|({S_1 - S_2})_{iau}| \\
\text{So, } \|\text{diag}([H(S_1)_{i} - H(S_2)_{i}])\|_F &= \sqrt{\sum_{i=1}^k \bigg(\sum_{j=1}^k H(S_1)_{ij} - H(S_2)_{ij}\bigg)^2} \\
&\leq \sqrt{\sum_{i=1}^k \bigg(\alpha_1^\prime \sum_{au}|({S_1 - S_2})_{iau}| \bigg)^2} \\
&= \alpha_1^\prime \|S_1 - S_2 \|_{1,1,2}
\end{align}


For the second term, $|f_2(H(S_1)) - f_2(H(S_2))|$

\begin{alignat}{2}
&= \frac{\mu_p}{b} \bigg| \text{tr}(H(S_1) - \text{diag}([H(S_1)_{i}]) ) &&- \text{tr}(H(S_2) - \text{diag}([H(S_2)_{i}]) ) \bigg| \\
&= \frac{\mu_p}{b} \bigg| \text{tr}(H(S_1) - H(S_2) ) &&-  \text{tr}\bigg(\text{diag}([H(S_1)_{i} - H(S_2)_{i}]) \bigg) \bigg| \\
&\leq \frac{\mu_p}{b} \bigg(\bigg| \text{tr}(H(S_1) - H(S_2) ) \bigg| &&+ \bigg| \text{tr}\bigg(\text{diag}([H(S_1)_{i} - H(S_2)_{i}]) \bigg) \bigg|\bigg) \\
&\leq \frac{\mu_p}{b} \bigg( \|\text{tr}\|\ \bigg\|H(S_1) - H(S_2)\bigg\|_F  &&+ \|\text{tr}\|\bigg\|\text{diag}([H(S_1)_{i} - H(S_2)_{i}]) \bigg\|_F \bigg| \bigg)\\
&\text{As shown above, }\\
&\leq \alpha_2 \|(S_1 - S_2)\|_F &&+ \alpha_2^\prime \|(S_1 - S_2)\|_{1,1,2}
\end{alignat}


For the third term, it has already been proven in \cite{zhao2012_consistency}, but we also prove it here:

$|f_3(H(S_1)) - f_3(H(S_2))|$

\begin{align}
&= \bigg| \frac{\text{tr}(H(S_1))}{\tilde{P_0}} - \frac{\sum_{i=1}^k (\sum_{j=1}^k H(S_1)_{ij})^2}{\tilde{P_0}^2} - \frac{\text{tr}(H(S_2))}{\tilde{P_0}} + \frac{\sum_{i=1}^k (\sum_{j=1}^k H(S_2)_{ij})^2}{\tilde{P_0}^2} \bigg| \\
&= \bigg| \frac{\text{tr}(H(S_1) - H(S_2))}{\tilde{P_0}} - \frac{\sum_{i=1}^k \big[H(S_1)_{i}^2 - H(S_2)_{i}^2\big] }{\tilde{P_0}^2} \bigg| \\
&\leq \frac{1}{\tilde{P_0}}\bigg| \text{tr}(H(S_1) - H(S_2))\bigg| + \frac{1}{\tilde{P_0}^2}\bigg| \sum_{i=1}^k \big[H(S_1)_{i}^2 - H(S_2)_{i}^2\big] \bigg| \\
&\text{As shown above, } \\
&\leq \alpha_3 \|(S_1 - S_2)\|_F + \frac{1}{\tilde{P_0}^2}\bigg| \sum_{i=1}^k \big[H(S_1)_{i} -  H(S_2)_{i}\big]\cdot\big[H(S_1)_{i} + H(S_2)_{i}\big] \bigg| \\
&\leq \alpha_3 \|(S_1 - S_2)\|_F + \frac{1}{\tilde{P_0}^2} \sum_{i=1}^k \big|H(S_1)_{i} + H(S_2)_{i}\big| \cdot \sum_{i=1}^k \big|H(S_1)_{i} - H(S_2)_{i}\big| \\
&= \alpha_3 \|(S_1 - S_2)\|_F + \alpha_3^\prime \sum_{iau} \big|S_1 - S_2\big|_{iau} \\
&= \alpha_3 \|(S_1 - S_2)\|_F + \alpha_3^\prime \|S_1 - S_2\|_{1,1,1}
\end{align}

\subsection{Required Condition b): Continuity of directional second derivative}
\label{supp:proof_consistency_conditions_2}
Condition 2: $W = H(\mathbb{D})$
\begin{align}
&\frac{\partial^2}{\partial\varepsilon^2}F(M_0 + \varepsilon(M_1 - M_0), \mathbf{t_0} + \varepsilon(\mathbf{t_1} - \mathbf{t_0}))\bigg|_{\varepsilon=0^+} \\
&= \frac{\partial^2}{\partial\varepsilon^2}f_1(M_0 + \varepsilon(M_1 - M_0)) + \frac{\partial^2}{\partial\varepsilon^2}f_2(M_0 + \varepsilon(M_1 - M_0)) + \frac{\partial^2}{\partial\varepsilon^2}f_3(M_0 + \varepsilon(M_1 - M_0)) \bigg|_{\varepsilon=0^+} 
\end{align}
Finding the directional derivative for $f_1$,
\begin{align}
f_1(M_0 + \varepsilon(M_1 - M_0)) &= \mu_p \text{tr}(X_C^T [M_0 + \varepsilon(M_1 - M_0) - \text{diag}(\sum_{j=1}^{k} \big(M_0 + \varepsilon(M_1 - M_0)\big)_{ij})] X_C) 
\\
&= \mu_p\bigg( \text{tr}(X_C^T M_0X_C) + \varepsilon\ \text{tr}(X_C^T (M_1 - M_0)X_C) \\ 
&\quad - \text{tr}(X_C^T \text{diag}([\sum_{j=1}^{k} (M_0)_{ij}])X_C) - \varepsilon\ \text{tr}(X_C^T \text{diag}([\sum_{j=1}^{k} (M_1 - M_0)_{ij}])X_C) \bigg) 
\\
\frac{\partial^2}{\partial\varepsilon}f_1(M_0 + \varepsilon(M_1 - M_0)) &= \mu_p\bigg(\text{tr}(X_C^T (M_1 - M_0)X_C) - \text{tr}(X_C^T \text{diag}([\sum_{j=1}^{k} (M_1 - M_0)_{ij}])X_C) \bigg) \\
\frac{\partial^2}{\partial\epsilon^2}f_1(M_0 + \varepsilon(M_1 - M_0) &= 0
\end{align}
Finding the directional derivative for $f_2$,
\begin{alignat}{2}
&f_2(M_0 + \varepsilon(M_1 - M_0))\\
&= \frac{\mu_p}{b} \text{tr}\bigg(M_0 + \varepsilon(M_1 - M_0) &&- \text{diag}(\sum_{j=1}^{k} \big(M_0 + \varepsilon(M_1 - M_0)\big)_{ij}) \bigg) - \frac{1}{b} + k - k\log b \\
&\frac{\partial}{\partial\varepsilon}f_2(M_0 + \varepsilon(M_1 - M_0)) &&= \frac{\mu_p}{b} \text{tr}(M_1 - M_0) - \frac{\mu_p}{b} \text{tr}\bigg( \text{diag}(\sum_{j=1}^{k} (M_1 - M_0)_{ij}) \bigg) \\
&\frac{\partial^2}{\partial\varepsilon^2}f_2(M_0 + \varepsilon(M_1 - M_0)) &&= 0
\end{alignat}
Finding the directional derivative for $f_3$,
\begin{align}
f_3(M_0 + \varepsilon(M_1 - M_0)) &= 
\frac{\text{tr}(M_0 + \varepsilon(M_1 - M_0))}{\tilde{P_0}} - \frac{\sum_{i=1}^k (\sum_{j=1}^k (M_0 + \varepsilon(M_1 - M_0))_{ij})^2}{\tilde{P_0}^2} \\
&= \frac{\text{tr}(M_0 + \varepsilon(M_1 - M_0))}{\tilde{P_0}} - \frac{\sum_{i=1}^k (\sum_{asu} \tilde{\pi}_s P_{as}x_u (M_0 + \varepsilon(M_1 - M_0))_{iau})^2}{\tilde{P_0}^2} \\
\frac{\partial}{\partial\varepsilon}f_3(M_0 + \varepsilon(M_1 - M_0)) &= \nonumber\\
\frac{\text{tr}(M_1 - M_0)}{\tilde{P_0}} + \sum_{i=1}^k &2(\sum_{asu} \tilde{\pi}_s P_{as}x_u (M_1 - M_0)_{iau}) \times\frac{\sum_{i=1}^k (\sum_{asu} \tilde{\pi}_s P_{as}x_u (M_0 + \varepsilon(M_1 - M_0))_{iau})^2}{\tilde{P_0}^2} \\
\frac{\partial^2}{\partial\varepsilon^2}f_3(M_0 + \varepsilon(M_1 - M_0)) &=
\frac{\sum_{i=1}^k 2(\sum_{asu} \tilde{\pi}_s P_{as}x_u (M_1 - M_0)_{iau})^2 }{\tilde{P_0}^2}
\end{align}
Adding up these three,
\[
\frac{\partial^2}{\partial\varepsilon^2}F(M_0 + \varepsilon(M_1 - M_0), \mathbf{t_0} + \varepsilon(\mathbf{t_1} - \mathbf{t_0})) = \frac{\sum_{i=1}^k 2(\sum_{asu} \tilde{\pi}_s P_{as}x_u (M_1 - M_0)_{iau})^2 }{\tilde{P_0}^2}
\]
which is continuous in $(M_1, \mathbf{t_1})$ for all $(M_0, \mathbf{t_0})$ in a neighborhood of $(W, \pi)$.

\subsection{Required Condition c): Upper bound of first derivative}
\label{supp:proof_consistency_conditions_3}

With $G(S) = F(H(S), h(S))$, $\frac{\partial G ((1-\varepsilon)\mathbb{D} + \varepsilon S)}{\partial \varepsilon}|_{\varepsilon=0^+} < -C < 0 \ \forall\ \pi, P$

$G(S) = f_1(H(S)) + f_2(H(S)) + f_3(H(S))$

Let  $\bar{S} = ((S - \mathbb{D})t)P(\mathbb{D}t)^T + (\mathbb{D}t)P(S-\mathbb{D}t)^T$

For $f_1$

\begin{align}
f_1(H((1-\varepsilon)\mathbb{D} + \varepsilon S)) &= \mu_p \text{tr}(X_C^T [H((1-\varepsilon)\mathbb{D} + \varepsilon S) - \text{diag}(\sum_{j=1}^{k} H((1-\varepsilon)\mathbb{D} + \varepsilon S)_{ij})] X_C) \\
H(S) &= (Sx)P(Sx)^T \\
H((1-\varepsilon)\mathbb{D} + \varepsilon S) &= ((1-\varepsilon)\mathbb{D}x + \varepsilon St)P((1-\varepsilon)\mathbb{D}x + \varepsilon St)^T \\
&= (\mathbb{D}x + \varepsilon (S - \mathbb{D})x)P(\mathbb{D}x + \varepsilon (S - \mathbb{D})x)^T \\
&= (\mathbb{D}x)P(\mathbb{D}x)^T + \varepsilon (\mathbb{D}x)P((S - \mathbb{D})x)^T \\&\quad\ + \varepsilon ((S - \mathbb{D})x)P(\mathbb{D}x)^T + \varepsilon^2 ((S - \mathbb{D})x)P((S - \mathbb{D})x)^T \\
\text{Finally, } f_1((1-\varepsilon)\mathbb{D} + \varepsilon S)
&= \text{tr}\big(X_C^T(\mathbb{D}x)P(\mathbb{D}x)^TX_C\big) + \varepsilon\ \text{tr}\big(X_C^T(\mathbb{D}x)P((S - \mathbb{D})x)^TX_C\big) \\&\quad\ + \varepsilon\ \text{tr}\big(X_C^T((S - \mathbb{D})x)P(\mathbb{D}x)^TX_C\big) + \varepsilon^2\ \text{tr}\big(X_C^T((S - \mathbb{D})x)P((S - \mathbb{D})x)^TX_C\big) \\
&\quad\ + 
\text{tr}\big(X_C^T\text{diag}([H(\mathbb{D})_i])X_C\big) +
\varepsilon^2\ \text{tr}\big(X_C^T\text{diag}([H(S-\mathbb{D})_i])X_C\big) \\&\quad\ + 
\varepsilon\ \text{tr}\big(X_C^T\text{diag}([\big(((S - \mathbb{D})x)P(\mathbb{D}x)^T + (\mathbb{D}x)P(S-\mathbb{D}x)^T \big)_i])X_C\big) \\
\text{Now,} \frac{\partial f_1}{\partial \varepsilon}\bigg|_{\varepsilon = 0^+}
	&= \text{tr}\big(X_C^T(\bar{S} - \text{diag}([\bar{S}_i]))X_C\big)
\end{align}

For $f_2$,

\begin{align}
f_2(H((1-\varepsilon)\mathbb{D} + \varepsilon S)) &= \frac{\mu_p}{b} \text{tr}\bigg(H((1-\varepsilon)\mathbb{D} + \varepsilon S) - \text{diag}([H((1-\varepsilon)\mathbb{D} + \varepsilon S)_i]) \bigg) - \frac{1}{b} + k - k\log b \\
\frac{\partial f_2}{\partial \varepsilon}\bigg|_{\varepsilon = 0^+} &= \frac{\mu_p}{b} \text{tr}\bigg(\bar{S} - \text{diag}([\bar{S}_i]) \bigg)
\end{align}

For $f_3$,

\begin{align}
f_3(H((1-\varepsilon)\mathbb{D} + \varepsilon S)) 
&= \frac{1}{\tilde{P}_0} \text{tr}\bigg( (\mathbb{D}x)P(\mathbb{D}x)^T + \varepsilon (\mathbb{D}x)P((S - \mathbb{D})x)^T \\
&\qquad \qquad\ + \varepsilon ((S - \mathbb{D})x)P(\mathbb{D}x)^T + \varepsilon^2 ((S - \mathbb{D})x)P((S - \mathbb{D})x)^T \bigg)\\
&- \frac{1}{\tilde{P}_0^2} \sum_{i=1}^k\sum_{j=1}^k\bigg(\bigg[ (\mathbb{D}x)P(\mathbb{D}x)^T + \varepsilon (\mathbb{D}x)P((S - \mathbb{D})x)^T \\
&\qquad \qquad\ + \varepsilon ((S - \mathbb{D})x)P(\mathbb{D}x)^T + \varepsilon^2 ((S - \mathbb{D})x)P((S - \mathbb{D})x)\bigg]_{ij} \bigg)^2 \\
\frac{\partial f_3}{\partial \varepsilon}\bigg|_{\varepsilon = 0^+} 
&= \frac{1}{\tilde{P}_0} \text{tr}\bigg( \bar{S} \bigg) - \frac{2}{\tilde{P}_0} \sum_{i=1}^k\bigg(\sum_{j=1}^k\bigg[ (\mathbb{D}x)P(\mathbb{D}x)^T\bigg]_{ij}\times \sum_{j=1}^k[\bar{S}]_{ij} \bigg)^2
\end{align}
From here, the proof is followed as in Appendix of \citet{zhao2012_consistency} (in Proof of Theorem 3.1)
 which also borrows from \citet{bickel2009_mod_consistency}.

\section{Dataset Summaries and Metrics}
\label{supp:datasets}

\begin{wraptable}{R}{0.5\linewidth}
  \centering
  \begin{tabular}{crrrr}
    \toprule
    Name    & p ($|\textbf{V}|$) & n ($|X_i|$) & e ($|E|$) & k ($y$) \\
    \midrule
    Cora    & 2708  & 1433 & 5278  & 7 \\
    CiteSeer& 3327  & 3703 & 4614  & 6 \\
    PubMed  & 19717 & 500  & 44325 & 3 \\
    \midrule
    Coauthor CS & 18333 & 6805 & 163788 & 15 \\ 
    Coauthor Physics & 34493 & 8415 & 495924 & 5 \\
    Amazon Photo & 7650 & 745 & 238162 & 8 \\
    Amazon PC & 13752 & 767 & 491722 & 10 \\
    ogbn-arxiv & 169343 & 128 & 1166243	& 40 \\
    \midrule
    Brazil  & 131   & 0 & 1074  & 4 \\
    Europe  & 399   & 0 & 5995  & 4 \\
    USA     & 1190  & 0 & 13599 & 4 \\
    \bottomrule
  \end{tabular}
  \caption{Datasets summary.}
  \label{datasets}
\end{wraptable}

Refer to Table \ref{datasets} for the dataset summary.

\textbf{Metrics.} 
A pair of nodes are said to be in agreement if they belong to the same class and are assigned to the same cluster, or they belong to different classes and have been assigned different clusters. For a particular partitioning, ARI is the fraction of agreeable nodes in the graph. Accuracy is obtained by performing a maximum weight bipartite matching between clusters and labels. NMI measures the normalized similarity between the clusters and the labels, and is robust to class imbalances. Mutual Information between two labellings $X$ and $Y$ of the same data is defined as
$ MI(X,Y) = \sum_{i=1}^{|X|} \sum_{j=1}^{|Y|} \frac{|X_i \cap Y_i|}{N} \log\frac{N|X_i \cap Y_i|}{|X_i||Y_i|} $
and it is scaled between 0 to 1. 

\section{Training Details}
\label{supp:training_details}
All experiments were run on an NVIDIA A100 GPU and Intel Xeon 2680 CPUs. We are usually running 4-16 experiments together to utilize resources (for example, in 40GB GPU memory, we can run 8 experiments on PubMed simultaneously). Again, the memory costs are more than dominated by the dataset. All experiments used the same environment running CentOS 7, Python 3.9.12, PyTorch 2.0, PyTorch Geometric 2.2.0.

\section{VGAE}
\label{supp:VGAE_eqns}
In a VGAE, the encoder learns mean ($\mathbf{\mu}$) and variance ($\mathbf{\sigma}$): 
\(
    \mu = \text{GCN}_\mu (X,A) \text{and}
    \log\sigma = \text{GCN}_\sigma (X,A)
\)
By using the reparameterization trick, we get the distribution of the latent space as:
\(
    q(Z|X,A) = \prod_{i=1}^N q(\mathbf{z_i}|X, A) = \prod_{i=1}^N \mathcal{N}(\mathbf{z_i}|\mu_i, \textrm{diag}(\sigma_i^2))
\)
A common choice for decoder is inner-product of the latent space with itself which giving us the reconstructed $\hat{A}$.
\(
    p(\hat{A}|Z) = \prod_{i=1}^p\prod_{j=1}^p p(\hat{A}_{ij}|z_i,z_j), \ \ \text{with} \ \ p(\hat{A}_{ij} = 1|z_i,z_j) = \textrm{sigmoid}(z_i^Tz_j)
\)

\section{Plots of evolution of latent space for other datasets}
\label{supp:latent_plots}

Refer to Figure \ref{fig:supp_latent}.
We can see the clusters forming in the latent space of the VGAEs. In the case of Q-VGAE, since a GCN is used on this space, it can learn non-linearities and the latent space shows different structures (like a starfish in CiteSeer). Moreoever, these structures have their geometric centres at the origin and grow out from there.
In contrast, for Q-GMM-VGAE, since a GMM is being learnt over the latent space, the samples are encouraged to be normally distributed in their independent clusters, all of which have different means and comparable standard deviations. So, we see multiple "blobs", which more or less follow a normal distribution. 
This plot effectively shows why a GMM-VGAE is more expressive than a VGAE.

\begin{figure}[!ht]
    \begin{subfigure}[t]{\linewidth}
        \centering
        \includegraphics[width=0.8\linewidth]{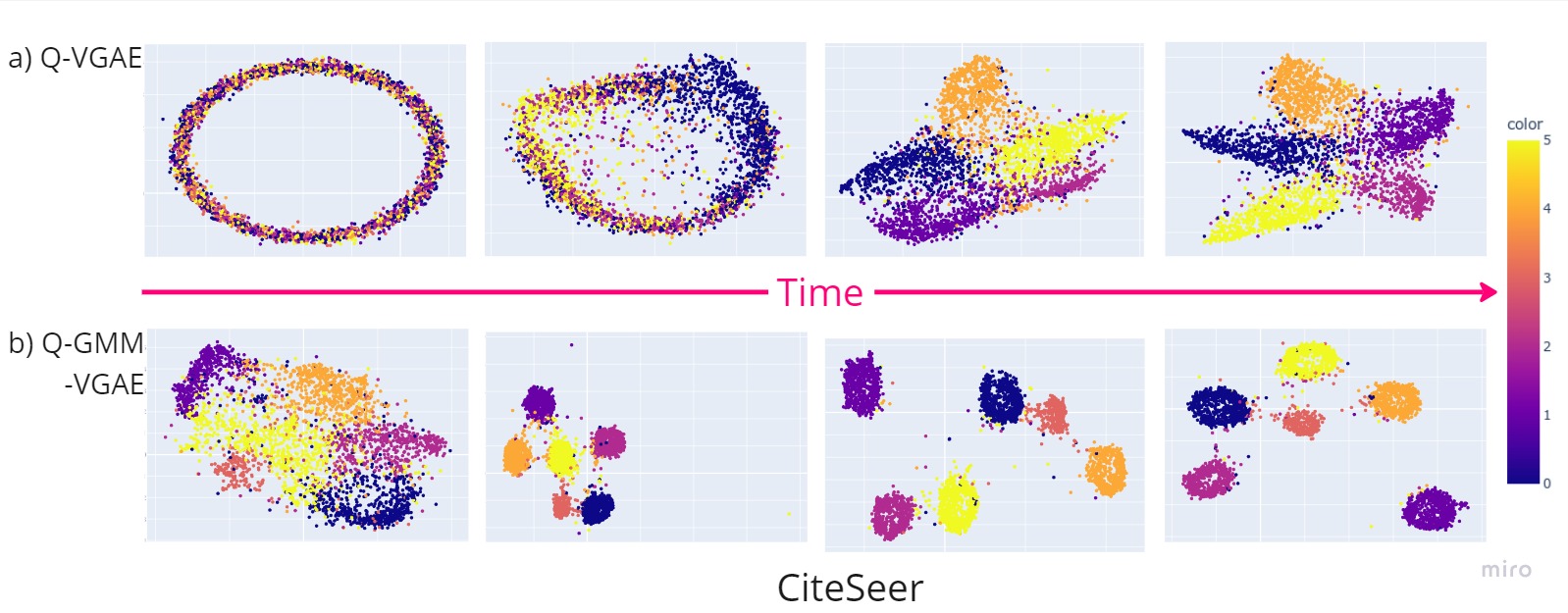}
    \end{subfigure}
    \begin{subfigure}[t]{\linewidth}
        \centering
        \includegraphics[width=0.8\linewidth]{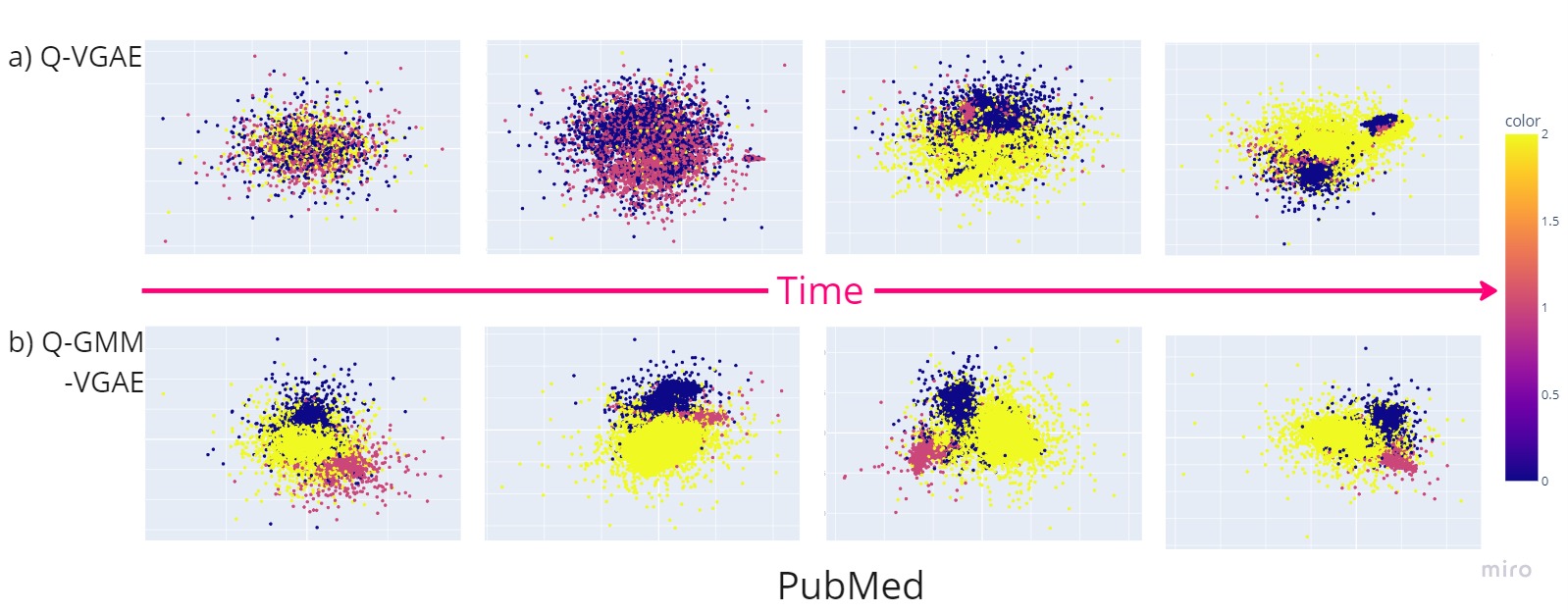}
    \end{subfigure}
    \begin{subfigure}[t]{\linewidth}
        \centering
        \includegraphics[width=0.8\linewidth]{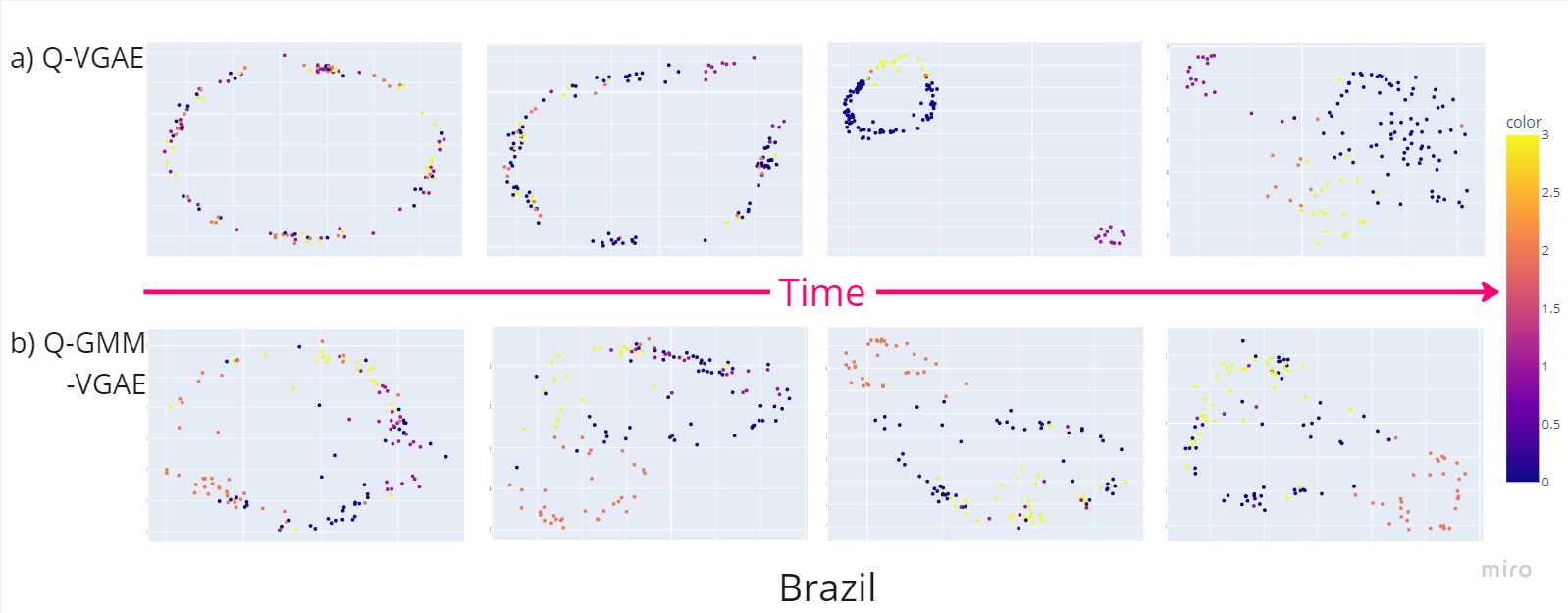}
    \end{subfigure}
    \begin{subfigure}[t]{\linewidth}
        \centering
        \includegraphics[width=0.8\linewidth]{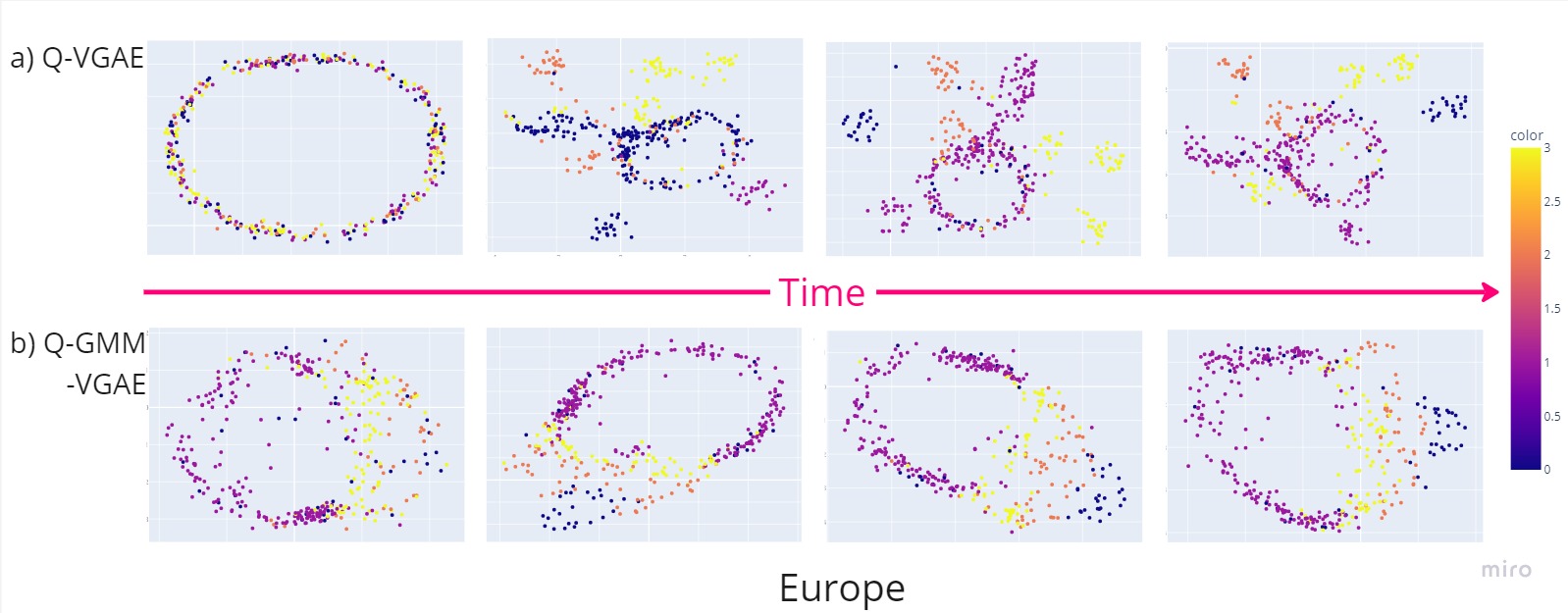}
    \end{subfigure}
    \caption{Plots of evolution of latent space for Q-VGAE and Q-GMM-VGAE methods for CiteSeer, PubMed, Brazil (Air Traffic) and Europe (Air Traffic) datasets.}
    \label{fig:supp_latent}
\end{figure}

\section{Attributed SBM theory and results}
\label{supp:sbm}
We validate the robustness and sensitivity of proposed methods to variance in the node features and graph structure. We are also generating features using a multivariate mixture generative model such that the node features of each block are sampled from normal distributions where the centers of clusters are vertices of a hypercube. 

\begin{figure}[hb]
    \centering
    \begin{subfigure}[t]{0.24\linewidth}
        \includegraphics[width=\linewidth]{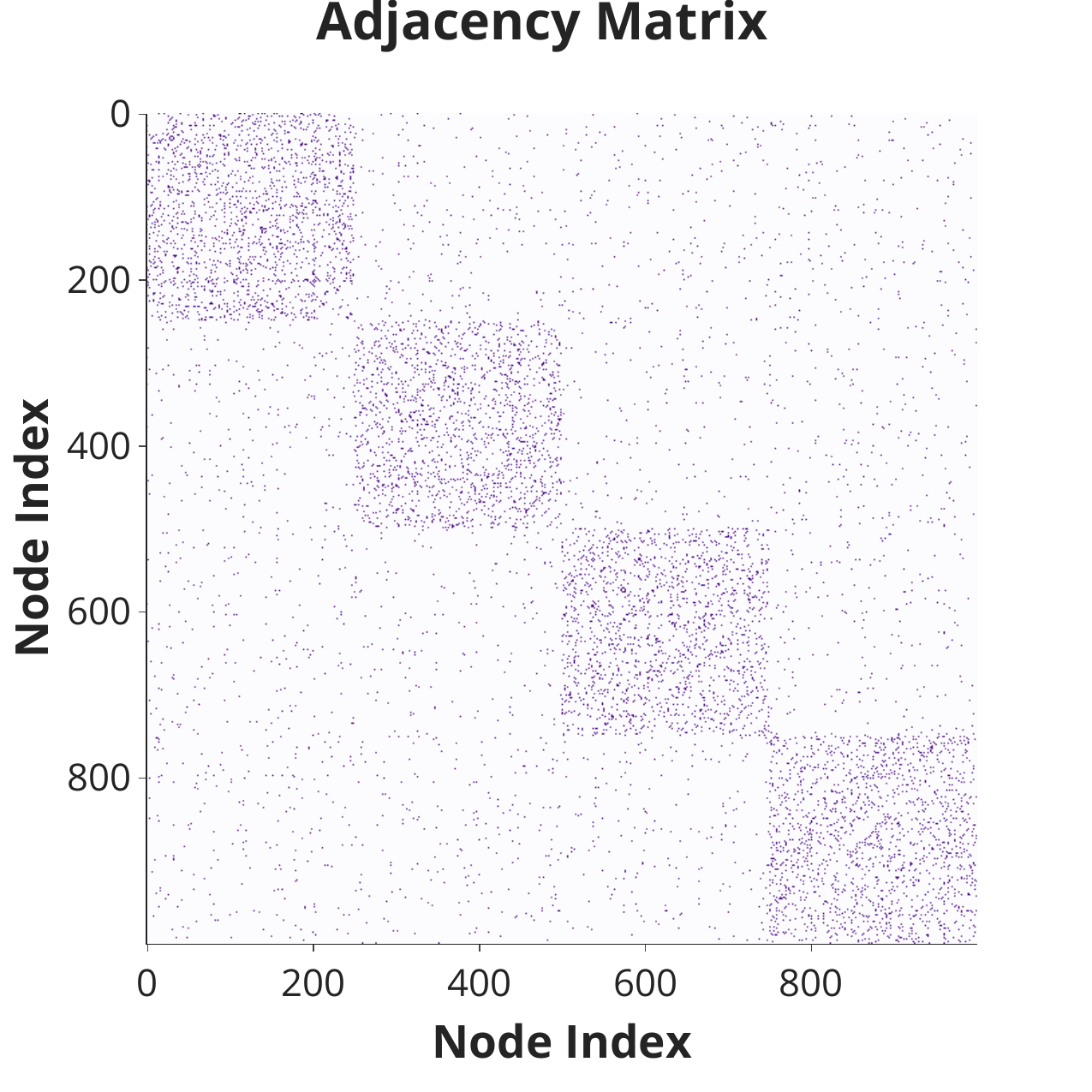}
        \caption{Adjacency Matrix}
    \end{subfigure}
    \begin{subfigure}[t]{0.24\linewidth}
        \includegraphics[width=\linewidth]{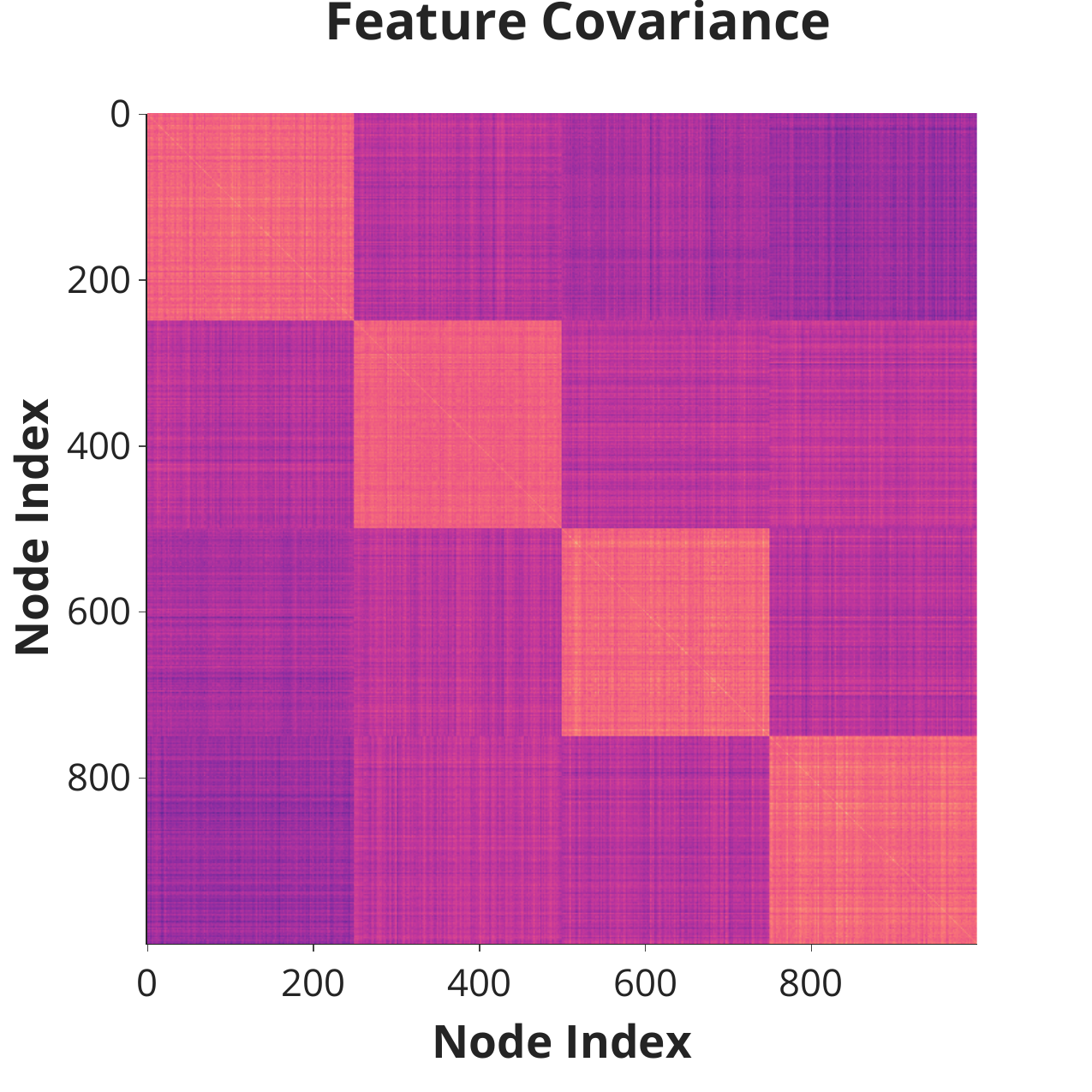}
        \caption{Feature Covariance\\ Matrix\\ $k_f = k$}
    \end{subfigure}
    \begin{subfigure}[t]{0.24\linewidth}
        \includegraphics[width=\linewidth]{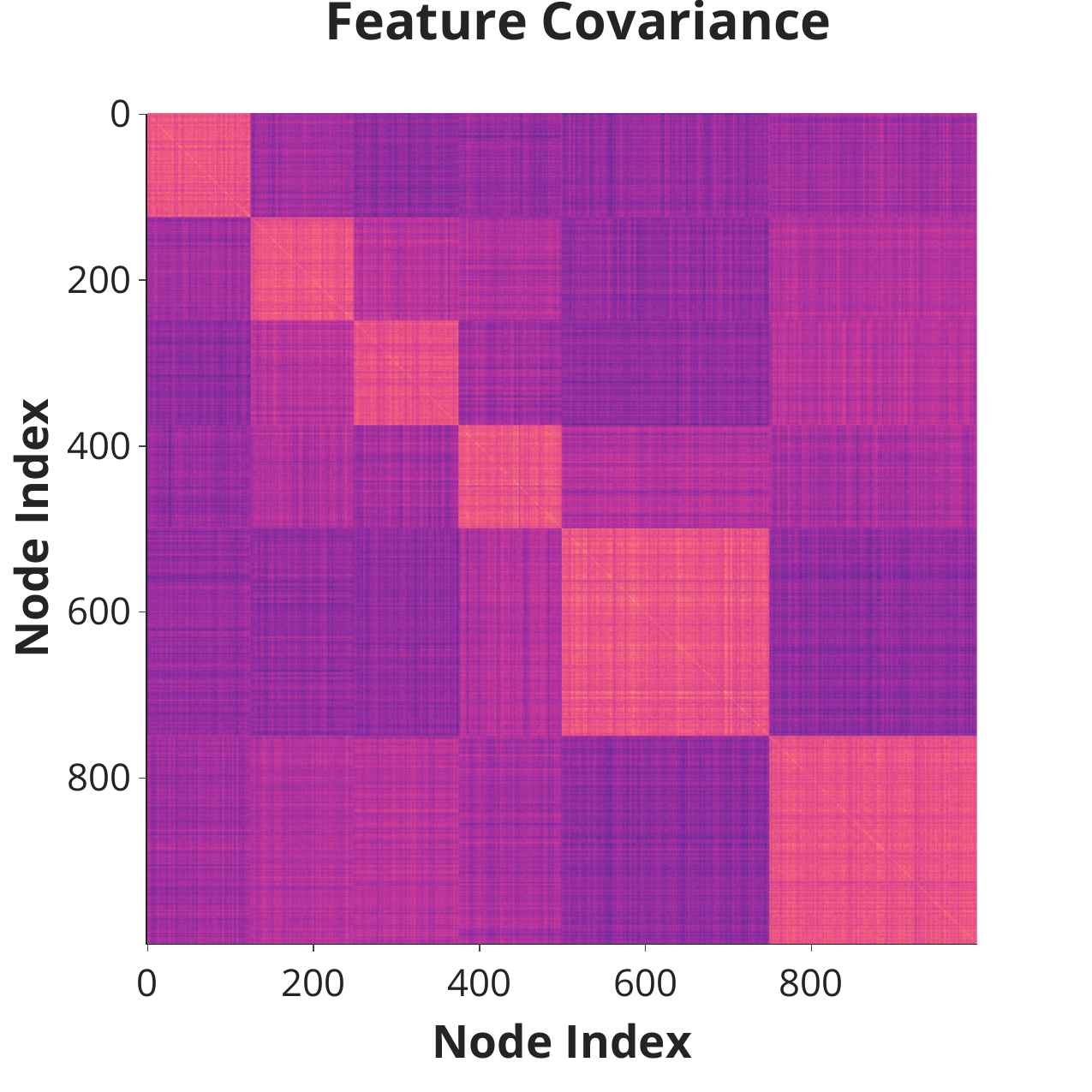}
        \caption{Feature Covariance\\ Matrix\\ $k_f > k$}
    \end{subfigure}
    \begin{subfigure}[t]{0.24\linewidth}
        \includegraphics[width=\linewidth]{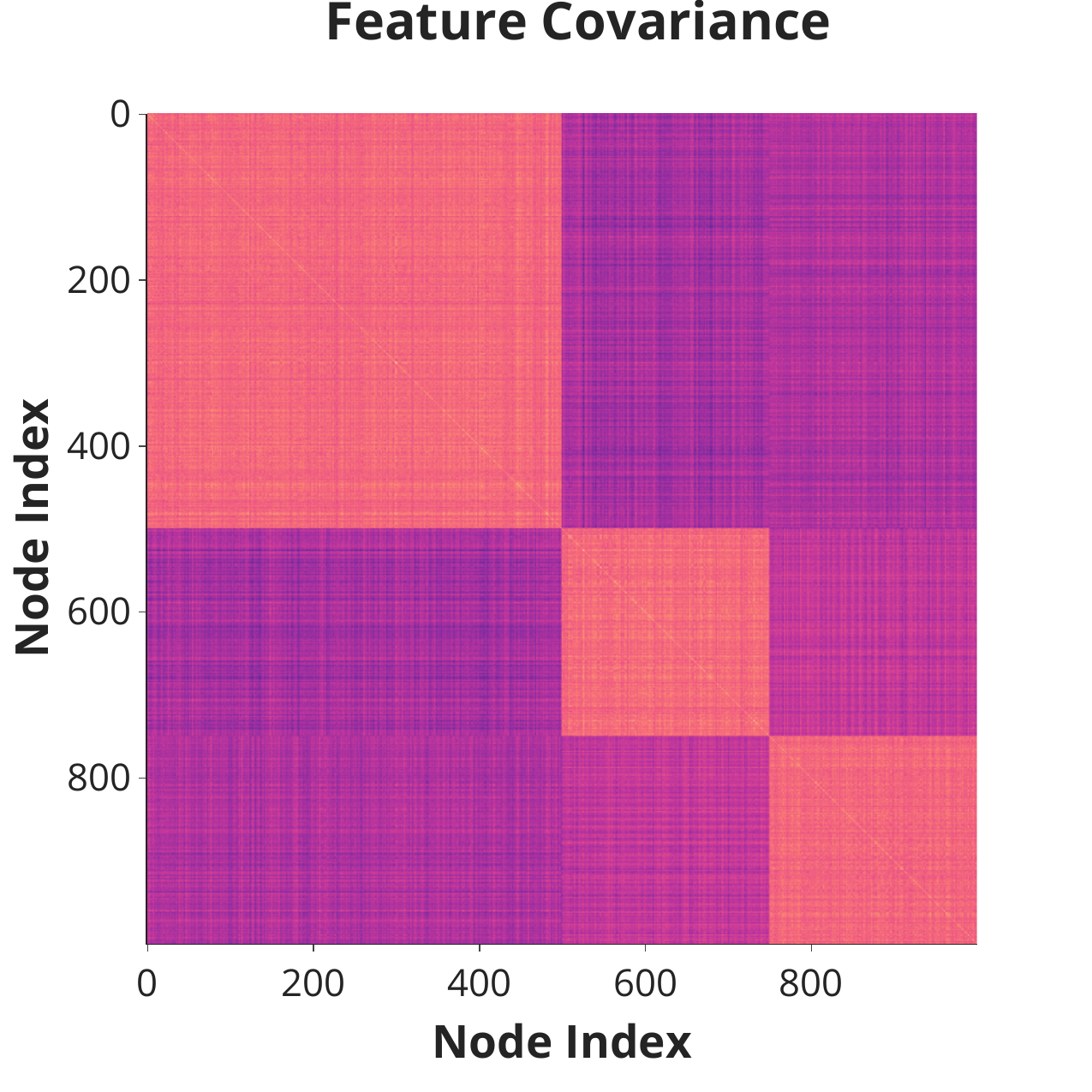}
        \caption{Feature Covariance\\ Matrix\\ $k_f < k$}
    \end{subfigure}
    \caption{Visualization of the generated adjacency and feature covariance matrices for the ADC-SBM}
    \label{fig:supp_sbm}
\end{figure}

\textbf{SBM.}
The Stochastic Block Model (SBM)\citep{NowickiSnijders2001SBM} is a generative model for graphs that incorporates probabilistic relationships between nodes based on their community assignments. In the basic SBM, a network with $p$ nodes is divided into $k$ communities or blocks denoted by $C_i$, where $i = 1, 2, \cdots, k$. 
The SBM defines a symmetric block probability matrix $B$ with size $(k \times k)$, where each entry $B_{ij}$ represents the probability of an edge between a node in community $C_i$ and a node in community $C_j$. Diagonal entries of this matrix represents the probabilities of intra-cluster edges. This matrix $B$ captures the intra- and inter-community connections and is assumed to be constant. 
$P(i \leftrightarrow j | C_i = a, C_j = b) = B_{ab}$ denotes the probability of an edge existing between nodes $i$ and $j$ when node $i$ belongs to community $a$ and node $j$ belongs to community $b$. Using these probabilities, the SBM generates a network by independently sampling the presence or absence of an edge for each pair of nodes based on their community assignments and the block probability matrix $B$.

\textbf{Degree Corrected SBM.}
DC-SBM\citep{KarrerNewman2011DCSBM} takes an extra set of parameters $\theta_i$ controlling the expected degree of vertex $i$. Now, the probability of an edge between two nodes (using the same notation as above) becomes $\theta_i\theta_jB_{ab}$. This was introduced to handle the heterogeneity of real-world graphs.

We selected this widely used and studied model for our analysis primarily because it considers a degree parameter for all nodes, resembling a key characteristic of real-world graphs. Therefore, the DC-SBM is closer to real-world graphs than the simpler SBM while still being analyzable, making it relevant here.

\textbf{ADC-SBM Generation.}
We make use of the \texttt{graph\_tool} library to generate the DC-SBM adjacency matrix, with $p=1000, k=4$. To generate the $B$ matrix, we follow the procedure in \citep{DMoN_JMLR_2023}, by taking expected degree for each node $d = 20$ and expected sub-degree $d_{out} = 2$. This gives us $B$ as: 
\[\begin{bmatrix}
    18 & 2 & 2 & 2 \\
    2 & 18 & 2 & 2 \\
    2 & 2 & 18 & 2 \\
    2 & 2 & 2 & 18 \\
\end{bmatrix}\]
Also, $\theta$ is generated by sampling a power-law distribution with exponent $\alpha=2$. We constrain the generated vector to $d_{min} = 2$ and $d_{max}=4$. 

To generate features, we use the \texttt{make\_classification} function in the \texttt{sklearn} library. We generate a $128$-dimensional feature vector for each node, with no redundant channels. These belong to $k_f$ groups, where $k_f$ might not be equal to $k$. We test three scenarios: a) matched clusters ($k_f=k$) b) nested features ($k_f>k$) c) grouped features ($k_f<k$) as visualized in Figure \ref{fig:supp_sbm}. Note that for better visualization, \texttt{class\_sep} was increased to 5 (however, the results are given with a value of 1, which is a harder problem).

Additionally, we also consider both cases, with and without the coarsening constraint term.

\textbf{Results.} Our objective is able to completely recover the ground truth labels (NMI/ARI/ACC = 1) under all the specified conditions. 

\section{Results on very large datasets}
\label{supp:large_dataset_results}
The results are presented in Table \ref{table:attributed_results_large}.

\begin{table}[ht]
    \centering
    \begin{adjustbox}{width=\linewidth}
    \begin{tabular}{lccc@{\hskip 20pt}ccc@{\hskip 20pt}ccc@{\hskip 20pt}ccc@{\hskip 20pt}ccc}
    \toprule[1.5pt]
    & \multicolumn{3}{c}{CoauthorCS} & \multicolumn{3}{c}{CoauthorPhysics} & \multicolumn{3}{c}{AmazonPhoto} & \multicolumn{3}{c}{AmazonPC} & \multicolumn{3}{c}{ogbn-arxiv} \\
    \cmidrule(r){2-4} \cmidrule(r){5-7} \cmidrule(r){8-10} \cmidrule(r){11-13} \cmidrule(r){14-16}
    \textbf{Method} & \textbf{ACC $\uparrow$} & \textbf{NMI $\uparrow$} & \textbf{ARI $\uparrow$} & \textbf{ACC $\uparrow$} & \textbf{NMI $\uparrow$} & \textbf{ARI $\uparrow$}  & \textbf{ACC $\uparrow$} & \textbf{NMI $\uparrow$} & \textbf{ARI $\uparrow$}  & \textbf{ACC $\uparrow$} & \textbf{NMI $\uparrow$} & \textbf{ARI $\uparrow$}  & \textbf{ACC $\uparrow$} & \textbf{NMI $\uparrow$} & \textbf{ARI $\uparrow$} \\ 
    \midrule[1.5pt]
        FGC & 69.6 & 70.4 & 61.5 & 69.9 & 60.9 & 49.5 & 44.9 & 38.3 & 22.5 & 46.8 & 36.2 & 23.3 & 24.1 & 8.5 & 9.1 \\
        \textbf{Q-MAGC (Ours)} & 70.2 & 76.4 & 60.2 & 75.3 & 67.2 & 66.1 & 70.4 & 66.6 & \textbf{58.6} & \textbf{62.4} & 51 & 31.1 & 35.8 & 24.4 & 15.6 \\
        \textbf{Q-GCN (Ours)} & 85.4 & 79.6 & 79.7 & 85.2 & \textbf{72} & \textbf{81.6} & 66.3 & 57.6 & 48.3 & 56.7 & 42.4 & 28.8 & 34.4 & 27.1 & 19.7 \\
        \textbf{Q-VGAE (Ours)} & \textbf{85.6} & \textbf{79.9} & \textbf{81.6} & \textbf{86.7} & 69 & 77.7 & 69.0 & 59.4 & 49.0 & 62.3 & 45.7 & \textbf{47.2} & \textbf{39.5} & 30.4 & \textbf{24.7} \\
        \textbf{Q-GMM-VGAE (Ours)} & 70.1 & 72.5 & 61.6 & 83.1 & 71.5 & 76.9 & \textbf{76.8} & \textbf{67.1} & 58.3 & 55.5 & \textbf{56.4} & 40 & OOM & OOM & OOM \\
        DMoN & 68.8 & 69.1 & 57.5 & 45.4 & 56.7 & 50.3 & 61.0 & 63.3 & 55.4 & 45.4 & 49.3 & 47.0 & 25.0 & \textbf{35.6} & 12.7 \\
    \bottomrule[1.5pt]
    \end{tabular}
    \end{adjustbox}
    \caption{Comparison of methods on large attributed datasets.}
    \label{table:attributed_results_large}
\end{table}

\section{Implementation}
\label{supp:code}
The implementations for all the experiments can be found at \url{https://anonymous.4open.science/r/MAGC-8880/}.

We have extensively used the PyTorch\citep{NEURIPS2019_9015_Pytorch} and PyTorch Geometric\citep{fey2019graph_pytorch_geometric} libraries in our implementations and would like to thank the authors and developers.

\section{Explanation on why VAE manifolds are curved}
\label{supp:vgae_manifolds}
Embedded manifolds obtained from VGAEs are curved and must be flattened before any clustering algorithms using Euclidean distance are applied. 

The latent space of a VAE is not constrained to be Euclidean. \citet{connor21a-vaelatentstructure} point out that the variational posterior is selected to be a multivariate Gaussian, and that the prior is modeled as a zero-mean isotropic normal distribution which encourages grouping of latent points around the origin. Works such as \citep{chen20i-vaeflatmanifolds, bogdanov2021-vaelatentspaces, arvanitidis2018latent} make the VAE latent space to be Euclidean/Hyperbolic/Riemannian, and show good visuals. Moreover, it can be observed in our own work (Figure \ref{fig:cora_latent}a and Appendix Figure \ref{fig:supp_latent}a) that the latent manifolds are curved and so, are not suitable for conventional methods such as k-means clustering, which need Euclidean distance.

\section{Complexities of some graph clustering methods}
\label{supp:common_complexities}
Some GCN\emph{-based} clustering methods:
\begin{itemize}[noitemsep]
    \item AGC\citep{agc2019} - $\mathcal{O}(p^2nt+ent^2)$\\
    where $t$ is the number of iterations (within an epoch)
    \item R-VGAE\citep{Mrabah2022RethinkingVGAE} - $\mathcal{O}(pk^2n+(p(n+k)+e(p+k))$
    \item S3GC\citep{s3gc_devvrit2022} - $\mathcal{O}(pn^2s)$\\ 
    where $s$ is the average degree
    \item HSAN\citep{HSAN} - $\mathcal{O}(pBn)$\\
    they state it as $\mathcal{O}(B^2d)$ but that is only for 1 batch of size $B$ and not the whole epoch
    \item VGAECD-OPT\citep{vgaecd-opt} - $\mathcal{O}(p^2n D^L)$\\ 
    where $D$ is the size of graph filter, $l$ is the number of linear layers
\end{itemize}

\section{Evolution of different loss terms throughout training}
\label{supp:loss_evol}
 Figure \ref{fig:loss_evol} - Each separate series has been normalized by its absolute minimum value to see convergence behavior on the same graph easily. Every series is decreasing/converging (except gamma, which represents sparsity regularization and remains almost constant). Thus, we can be assured that no terms are counteracting and hurting the performance. The legend is provided in the graph itself. This plot is on the Cora dataset.

\begin{figure}
    \centering
    \includegraphics[width=\linewidth]{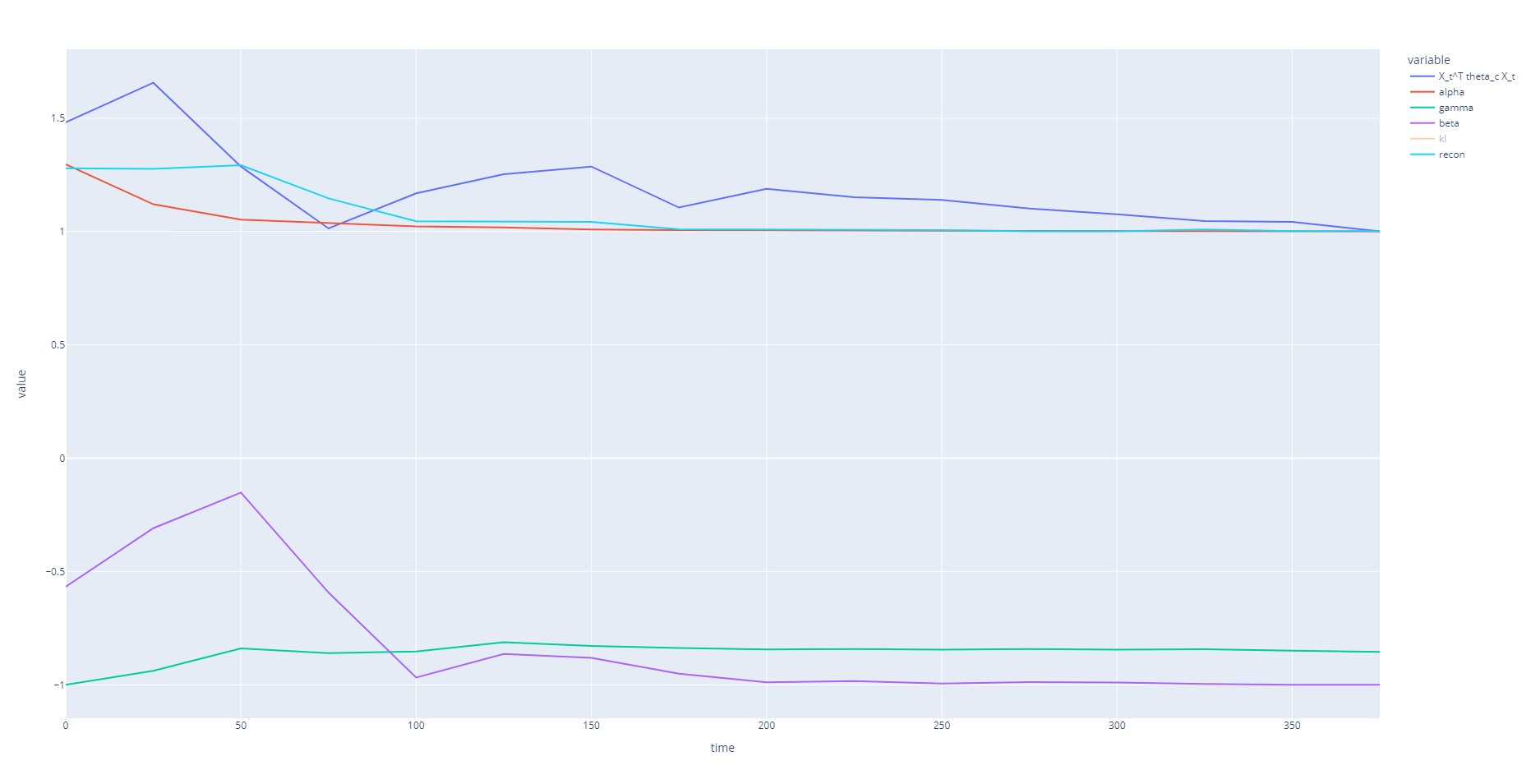}
    \caption{Evolution of the different loss terms throughout training, denoted by their weight parameters. Also the term \texttt{X\_t\^{}T theta\_C X\_t} term is the smoothness term $tr(X_C^T C^T \Theta C X_C)$}
    \label{fig:loss_evol}
\end{figure}

\end{document}

%% file: math_commands.tex
\usepackage{amsmath,amsfonts,bm}

\def\eqref#1{equation~\ref{#1}}
\def\Eqref#1{Equation~\ref{#1}}

\def\1{\bm{1}}

\DeclareMathAlphabet{\mathsfit}{\encodingdefault}{\sfdefault}{m}{sl}
\SetMathAlphabet{\mathsfit}{bold}{\encodingdefault}{\sfdefault}{bx}{n}

